\PassOptionsToPackage{table}{xcolor}
\documentclass{article}
\usepackage{iclr2027_conference,times}
\iclrfinalcopy
\usepackage{amsmath,amssymb,amsthm}
\usepackage{booktabs}
\usepackage{graphicx}
\usepackage{float}
\usepackage{microtype}
\usepackage{multirow}
\usepackage{needspace}
\usepackage{tabularx}
\usepackage{xcolor}
\usepackage{hyperref}
\usepackage[capitalise,noabbrev]{cleveref}
\usepackage{url}

\newcommand\hmmax{0}
\newcommand\bmmax{0}

\usepackage{amsmath,amsfonts,bm}

\def\eqref#1{equation~\ref{#1}}

\def\1{\bm{1}}

\DeclareMathAlphabet{\mathsfit}{\encodingdefault}{\sfdefault}{m}{sl}
\SetMathAlphabet{\mathsfit}{bold}{\encodingdefault}{\sfdefault}{bx}{n}

\newcommand{\E}{\mathbb{E}}

\newcommand{\cC}{\mathcal{C}}

\newcommand{\cL}{\mathcal{L}}

\newcommand{\method}{\textsc{CT-OPD}}
\newcommand{\model}{\textsc{DiMoE-VL}}
\newcommand{\modelbase}{\textsc{DiMoE-VL-Base}}

\newcommand{\teacher}{\mathrm{T}}
\newcommand{\student}{\mathrm{S}}
\newcommand{\masktok}{\mathtt{[M]}}
\definecolor{tablehead}{HTML}{DCE8F2}
\definecolor{tablegroup}{HTML}{EEF1F4}
\definecolor{tablebase}{HTML}{F7F8F9}
\definecolor{tableours}{HTML}{EAF3FA}
\definecolor{tablebest}{HTML}{D8EAF7}
\definecolor{tablegain}{HTML}{B5403A}
\definecolor{tableloss}{HTML}{2F6FA3}
\definecolor{tablemuted}{HTML}{66717B}
\definecolor{tableline}{HTML}{8895A1}
\newcommand{\ctgain}[2]{#1\,\textcolor{tablegain}{\scriptsize (#2)}}
\newcommand{\ctbestgain}[2]{{\bfseries #1\,\textcolor{tablegain}{\scriptsize (#2)}}}
\newcommand{\ctloss}[2]{#1\,\textcolor{tableloss}{\scriptsize (#2)}}

\DeclareMathOperator{\Tok}{Tok}
\DeclareMathOperator{\Detok}{Detok}

\IfFileExists{generated/LLADAV_EXPLICIT_FINAL_HEADLINE.tex}{

}{}
\IfFileExists{generated/CROSS_ARCHITECTURE_ACC_LOOSE_V2_COMPLETE.tex}{
    \input{generated/CROSS_ARCHITECTURE_ACC_LOOSE_V2_HEADLINE.tex}
}{}
\IfFileExists{generated/LLADAV_CT_OPD_MAIN_COMPLETE.tex}{

}{}
\IfFileExists{generated/DIMOE_CT_OPD_RAW_BASELINE_COMPLETE.tex}{

}{}

\newtheorem{proposition}{Proposition}
\crefname{proposition}{proposition}{propositions}

\title{CT-OPD: Counterfactual Trace\\
On-Policy Distillation for\\
Diffusion Vision-Language Models}

\author{
Long Qian$^{1,2}$,
Bingke Zhu$^{1,2}$\thanks{Corresponding author.},
Jiaqi Wei$^{1,3}$\footnotemark[2],
Yu Li$^{1,4}$,
Yingying Chen$^{1,2}$,
Jinqiao Wang$^{1,2,5}$
\\
$^{1}$Foundation Model Research Center, Institute of Automation, Chinese Academy of Sciences
\\
$^{2}$School of Future Technology, University of Chinese Academy of Sciences
\\
$^{3}$School of Engineering, Cardiff University
\\
$^{4}$School of Advanced Interdisciplinary Sciences, University of Chinese Academy of Sciences
\\
$^{5}$Wuhan AI Research
\\
\texttt{\{qianlong2024,liyu2025\}@ia.ac.cn,}
\texttt{WeiJ16@cardiff.ac.uk,}
\\
\texttt{\{bingke.zhu,yingying.chen,jqwang\}@nlpr.ia.ac.cn}
}

\begin{document}
\raggedbottom

\maketitle
\begingroup
\renewcommand{\thefootnote}{\fnsymbol{footnote}}
\footnotetext[2]{Work done during internships at Institute of Automation, Chinese Academy of Sciences.}
\endgroup
\begin{abstract}
Diffusion vision-language models generate answers by gradually resolving masked
tokens, making accurate conditional prediction in partially resolved states
central to post-training. Masking completed answers yields coherent contexts
and targets, but prescribed masks do not reflect the model's reveal decisions.
Its trajectories capture these decisions, yet their provisional visible tokens
can conflict with the target response. Outcome-based reinforcement learning
follows these trajectories but provides only response-level feedback, which
loses contrast when sampled rewards tie.
To align coherent token-level supervision with the model's reveal decisions,
we introduce Counterfactual Trace On-Policy Distillation (\method{}), which
combines completed teacher responses with trajectory masks from the current
student. CT-OPD retokenizes each teacher response in the student's vocabulary
and extracts unresolved-position masks at successive stages of the student's
reverse process. For each mask, it discards provisional rollout values and
reconstructs the partial state from the teacher endpoint, so the supervised
positions follow the current trajectory while the visible context and targets
remain consistent with the same response. The student is trained on these
reconstructed states with its native categorical loss, and trajectories are
refreshed as the model evolves.
Across dense and sparse diffusion architectures, CT-OPD consistently enhances
multimodal understanding and reasoning capabilities, with gains of up to 9.80
points on the nine-benchmark average. On the unified understanding-and-generation architecture, it
also improves both visual understanding and image generation, showing that the
same principle transfers across architectures and
modalities. Ablations further attribute these gains to coherent reconstruction
and current-model trajectory masks.

\end{abstract}

\section{Introduction}
\label{sec:introduction}

Diffusion vision-language models (VLMs) generate responses by repeatedly
predicting masked token positions. Beginning with a fully masked canvas, the
reverse process gradually reveals an answer. At every stage, unresolved tokens
are inferred from the image, the question, and the portion of the response
already available \citep{you2025lladav,li2025lavida,yang2025mmada}. Unlike
left-to-right decoding, this process revisits the answer as its visible content
changes, so each step presents a different conditional prediction problem.
Effective post-training therefore requires accurate supervision across the
partial states encountered throughout this reverse process.

Existing post-training methods address different aspects of this requirement.
As shown in Fig.~\ref{fig:motivation}, supervised fine-tuning or distillation starts from a high-quality completed
response, masks some of its tokens, and trains the model to recover them.
Because the visible tokens and prediction targets come from the same response,
the resulting task is precise and internally consistent. However, its masks
are chosen independently of the current model: a position that the model
repeatedly leaves unresolved receives no special attention over one it already
resolves reliably.  Outcome-based reinforcement learning follows the model's
own completions, but a score assigned after the response is finished cannot
identify which intermediate prediction caused an error or which token should
replace it.  This loss of detail is especially consequential for weak models.
In our 512-prompt LLaDA-V probe as shown in Tab.~\ref{tab:zero_contrast_learning}, 60.74\% of four-response groups receive identical
correctness rewards, leaving centered group-relative objectives with no task
advantage even though token-level corrections are available from a strong
teacher.  Autoregressive (AR) on-policy distillation offers another piece of the
solution by querying a teacher on model-generated prefixes
\citep{agarwal2023opd}.  Its causal query, however, is different from the
bidirectional partial canvas native to diffusion decoding, and teacher logits
do not align directly when the two models use different tokenizers.

Together, these limitations reveal what a diffusion post-training state
must contain: the positions to supervise should reflect the model's current
reverse trajectory, while the visible context and token targets should
describe one coherent response. These ingredients come from different
sources: a completed response provides the tokens to form consistent
context and targets, whereas a rollout trace records the model's reveal
decisions through its mask pattern.
Yet the trace carries provisional visible tokens, which serve as context
for predicting the unresolved positions. Transferring the trace intact and
replacing only its targets can therefore create a mismatch: whenever a visible
trace token differs from the completed response, the model sees part of one
response and is asked to complete another. Drawing a fresh mask from a
prescribed schedule over the completed response removes this mismatch, but
also discards the model-specific reveal pattern that makes the trace useful.
We therefore retain only the unresolved-position pattern as a
\emph{trajectory mask} and place it over the completed response, or
\emph{endpoint}. Within that response, the mask identifies the positions left
unresolved by the current trajectory, while the endpoint supplies both visible
context and targets. The reconstructed partial state thus combines
trajectory-aware supervision with
one coherent conditional prediction problem.

\begin{figure}[t]
    \centering
    \includegraphics[width=\textwidth]{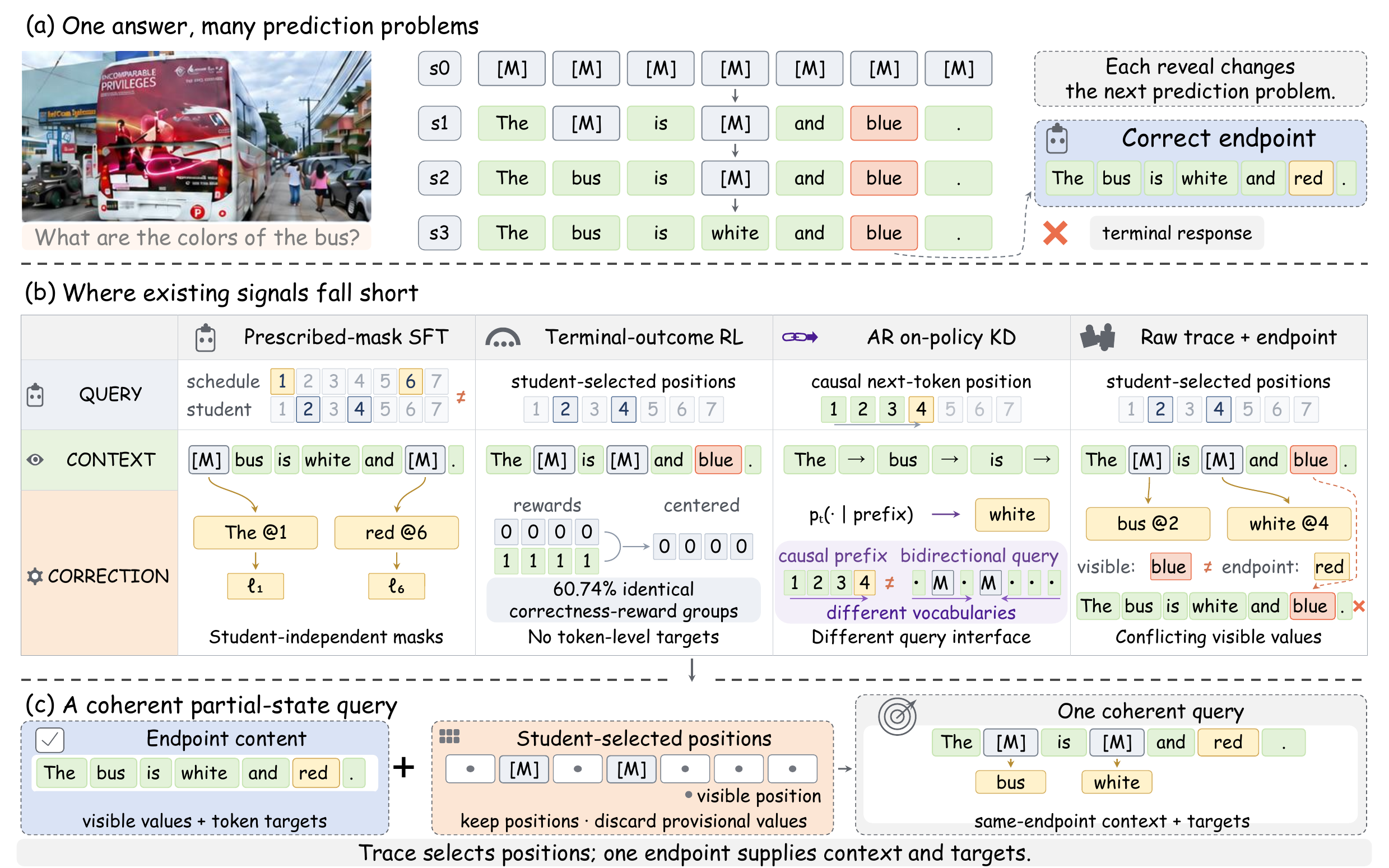}
    \caption{\textbf{Motivation of CT-OPD.} Completed responses provide coherent
    content but do not specify the student's unresolved positions. Student
    trajectories reveal these positions, yet their provisional tokens may
    conflict with endpoint targets. CT-OPD combines trajectory-selected
    positions with context and targets from one endpoint.}
    \label{fig:motivation}
\end{figure}

To address these challenges, we propose \textbf{Counterfactual Trace On-Policy
Distillation (\method{})}, a post-training method that builds coherent
partial-response tasks along the current model's reverse trajectory. CT-OPD
serializes a completed teacher response, retokenizes it with the
student's tokenizer, and adapts the result to the student's endpoint
convention. At the beginning of each training cycle, the current student
completes a roolout on the same image-question prompt, yielding unresolved-position
masks at the fully masked, early, middle, and late stages of decoding. CT-OPD
discards the rollout's provisional token values and overlays each mask on the
retokenized endpoint: endpoint tokens fill the visible positions, while
endpoint tokens under the mask provide the prediction targets. Each
reconstructed state follows the student's current reveal pattern
while presenting a context and correction drawn from one response. The
student optimizes these four states sequentially with its native categorical
diffusion loss, then generates a fresh trajectory for the next cycle. Since
the procedure transfers student-tokenized endpoints and trajectory masks
instead of teacher logits, it applies across mismatched tokenizers,
diffusion architectures, and both text and image token streams.

Across dense and sparse diffusion architectures, CT-OPD strengthens multimodal
understanding and reasoning, improving the nine-task average by up to 9.80
points over the baseline and consistently outperforming existing post-training
methods. Comparisons with random-mask diffusion SFT separate the benefit of
partial-state training from trajectory-based position selection, while
controlled studies further establish the roles of coherent reconstruction and
current, prompt-matched masks. To further evaluate the generality of CT-OPD beyond visual understanding, we
apply it to unified MMaDA. The same construction improves both understanding
and image generation, demonstrating that CT-OPD transfers across architectures
and output modalities. Our contributions
are:
\begin{itemize}
    \item We formulate diffusion-VLM post-training as the construction of
    coherent partial-response prediction tasks, separating response content
    from trajectory-dependent position selection and exposing the
    context-target mismatch introduced by retaining provisional rollout values.
    \item We introduce CT-OPD, an on-policy distillation method that turns
    teacher responses and student trajectories into coherent token-level
    supervision, adapting training tasks to the evolving student across
    tokenizers, diffusion architectures, and output modalities.
    \item We demonstrate consistent performance across dense,
    sparse, and unified diffusion VLMs. CT-OPD achieves the best nine-benchmark
    understanding average on every evaluated backbone, raises LLaDA-V by 9.80
    points over the baseline, and improves all three image-generation metrics
    on MMaDA, while ablations verify each component of the design.
\end{itemize}

\section{Related Work}
\label{sec:related}

\paragraph{Diffusion vision-language models.}
Masked-diffusion language models predict unresolved tokens from bidirectional
context \citep{nie2025llada}, a formulation extended to multimodal understanding
and generation by LLaDA-V, LaViDa, MMaDA, and later systems
\citep{you2025lladav,li2025lavida,yang2025mmada,zeng2025diffusionvl,wu2026fastdvlm,chen2026bard}.
Mixture-of-experts models scale capacity through conditional computation
\citep{fedus2021switch}, also used in AR VLMs and text diffusion
\citep{lin2024moellava,zhu2025lladamoe}.  In sparse \model{}, CT-OPD's
reconstructed partial state drives both expert selection and token prediction
under one native loss.

\paragraph{Reinforcement learning for diffusion models.}
Diffusion post-training uses reward-weighted denoising, reverse-state policy
gradients, intermediate branching, trajectory balance, and multimodal tree search
\citep{zhu2025dmpo,he2025mdpo,oba2026dispo,ahmadi2026trafl,li2026lavidar1}.
TraceRL, d3LLM, and MaskGRPO further exploit preferred trajectories, decoding
orders, and grouped outcomes \citep{wang2026tracerl,qian2026d3llm,ma2026maskgrpo}.
CT-OPD uses the reverse process to select unresolved positions, then couples
them to a verified endpoint for dense token supervision when group rewards
coincide and centered task advantages vanish.

\paragraph{On-policy and cross-tokenizer distillation.}
AR on-policy distillation queries a teacher on student-generated prefixes
\citep{agarwal2023opd}, while sequence-level and cross-tokenizer methods
transfer complete responses or align distributions across vocabularies
\citep{kim2016sequencekd,boizard2024uld,minixhofer2025alm,zhang2025dskd,wang2026bpm}.
For diffusion students, OPDLM derives causal prefixes from student endpoints, and
TOPD matches masked-state distributions from a compatible teacher, and dOPSD
constructs privileged targets from later trajectory states, and OPTD
compresses future reverse transitions for few-step decoding
\citep{su2026opdlm,ren2026topd,dat2026dopsd,lu2026optd}.  CT-OPD instead
retokenizes an answer-verified AR VLM response for visible context and targets,
using student traces to select the masked positions.

\section{Method}
\label{sec:method}

CT-OPD forms each training state from two pieces with different roles.  A
teacher endpoint determines the response to be learned, while a
\emph{trajectory mask}-the positions left unresolved at one stage of the
current model's reverse process-determines where that response is predicted.
Expressing the endpoint in the student's vocabulary and reconstructing a
partial state with its tokens under this mask allows completed teacher answers
to supervise the model's own sequence of reveal stages.

\subsection{Teacher endpoints and trajectory masks}

\begin{figure}[t]
    \centering
    \includegraphics[width=\textwidth]{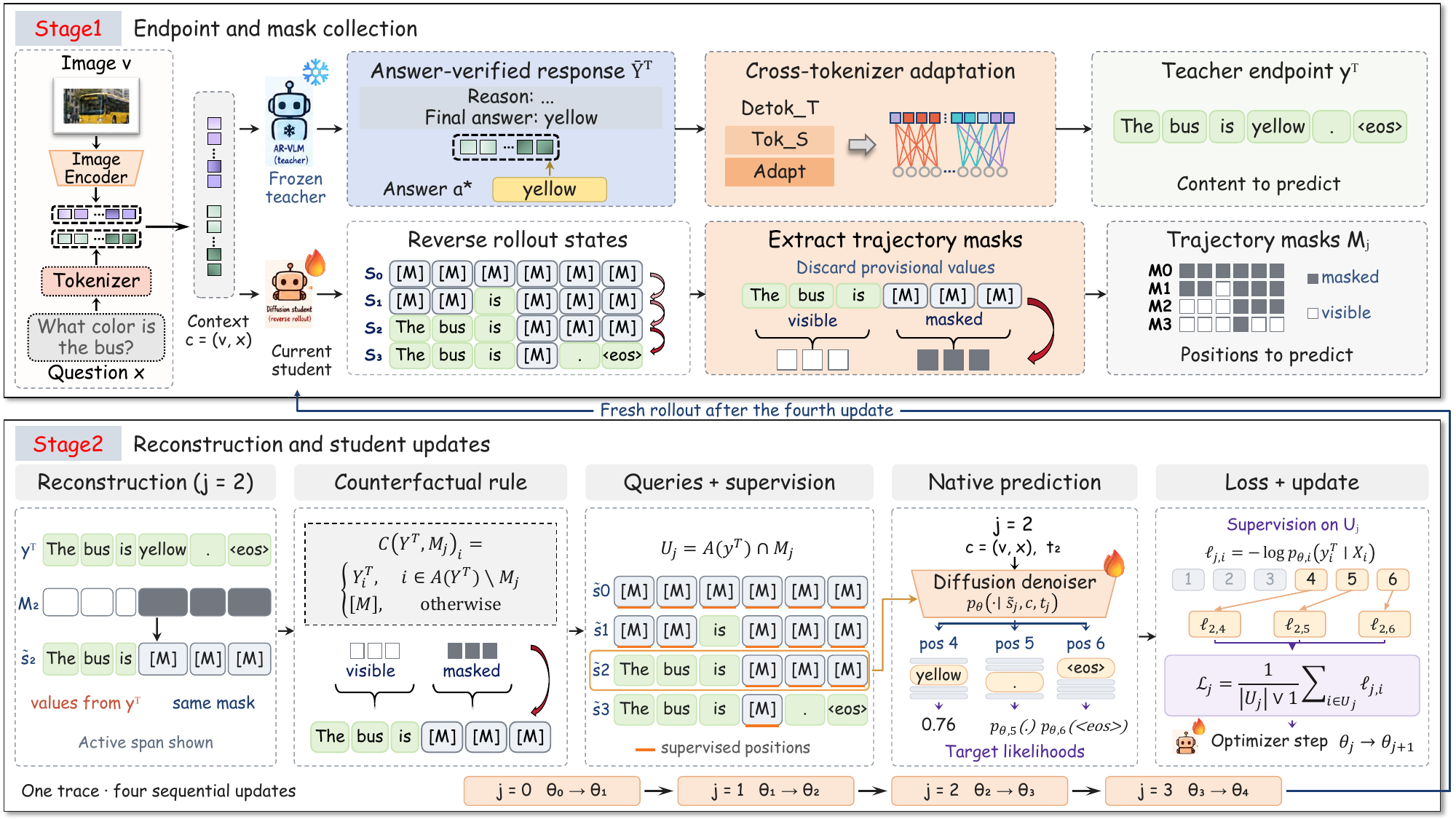}
    \caption{\textbf{Overview of CT-OPD.} A completed teacher response is
    retokenized as the student endpoint, while a current-student rollout supplies
    trajectory masks at four denoising stages.  CT-OPD overlays each mask on the
    endpoint to form coherent partial states and supervises unresolved
    endpoint tokens with the native categorical loss. The updated student
    generates a fresh trajectory for the next cycle.}
    \label{fig:method}
\end{figure}

A teacher answer needs adaptation when tokenization or termination conventions
differ. Given image $v$, question $x$, and training answer $a^*$, let
$c=(v,x)$. The AR VLM generates response $\bar Y^\teacher$ with a rationale
and verified final answer. We serialize and retokenize it, then apply the
endpoint adapter $\operatorname{Adm}_a$ for student architecture $a$:
\begin{equation}
Y^\teacher=\operatorname{Adm}_a\!\left[
\Tok_\student\!\left(\Detok_\teacher(\bar Y^\teacher)\right)\right].
\label{eq:cross_tokenizer}
\end{equation}
$\operatorname{Adm}_a$ fits the canvas by trimming the rationale as needed,
retaining the verified answer suffix and appending the native end-of-turn
token. Adapted endpoints have law
$\pi_\teacher(dy\mid c,a^*)$ and define the active span scored through the
student's native head without teacher logits on masked states.

With the endpoint in student tokens, we choose where to place its masks.
A prescribed schedule does not reflect the student's reveal decisions, so each
cycle freezes the parameters as $\bar\theta$ and samples a reverse trace $Z$.
The selector $G_j$ extracts the unresolved positions at stage $j$:
\begin{equation}
Z\sim\mathbb P_{\bar\theta}^{\mathrm{rev}}(\cdot\mid c),
\qquad M_j=G_j(Z),
\qquad \rho_{\bar\theta,j}(m\mid c)=\Pr(M_j=m\mid c).
\label{eq:trajectory_mask_distribution}
\end{equation}
Each mask $M_j$ lies on the common length-$L$ canvas
$\mathcal I=\{1,\ldots,L\}$.  If native scoring requires metadata $H_j$ that
is not determined by $M_j$, we write $\Xi_j=(M_j,H_j)$ for the complete
trajectory descriptor, and otherwise $\Xi_j=M_j$.  Set operations below act only
on $M_j$, and the rollout's tentative response values are discarded.

\Needspace{5\baselineskip}
The trajectory mask depends on the question and current student, so it stays
with its originating prompt.  Given a training example, the teacher endpoint
and raw trace are sampled independently and paired by that prompt, and the trace's
provisional response values play no role in their pairing.

\subsection{Constructing coherent partial states}

A trajectory mask identifies missing positions, while the raw trace from which
it was extracted may contain visible values that disagree with the teacher, and
its canvas may also extend beyond the retokenized endpoint.  Let
$A(y)\subseteq\mathcal I$ contain the endpoint's active coordinates, including
its supervised terminal token.  We project a raw mask
$m\subseteq\mathcal I$ onto this span and fill every visible active coordinate
from $y$:
\begin{equation}
\cC(y,m)_i=
\begin{cases}
y_i,&i\in A(y)\setminus m,\\
\masktok,&\text{otherwise},
\end{cases}
\qquad U(y,m)=A(y)\cap m.
\label{eq:active_transport}
\end{equation}
Thus $U(y,m)$ contains exactly the active positions to score, and coordinates
outside $A(y)$ remain masked and receive no loss.  Replacing raw visible values
does not change the positions selected by the trajectory mask within the
endpoint span.  It changes the question asked there: the student now completes
the same response that supplies its targets.  These requirements uniquely
determine $\cC(y,m)$ (\Cref{prop:occupancy_projection}), which the native
adapter maps to a valid query for each architecture.

At fixed context, answer, and stage, an endpoint and current-model trajectory
mask are sampled independently. Reconstruction defines their joint law
$\widetilde d_{\bar\theta,j}$ with the partial state:
\begin{equation}
\widetilde d_{\bar\theta,j}(dy,dm,ds\mid c,a^*)
=\pi_\teacher(dy\mid c,a^*)\,
\rho_{\bar\theta,j}(dm\mid c)\,
\delta_{\cC(y,m)}(ds).
\label{eq:projected_occupancy_law}
\end{equation}
The product law keeps endpoint content and mask placement distinct, while
$s=\cC(y,m)$ makes their output one prediction task. The mask determines target
coordinates and which endpoint tokens are visible.
Relative to prescribed time-$t$ corruption $\nu_t(m\mid y,c)$ in diffusion
SFT, the queried positions follow the current student's unresolved-set
occupancy and all
visible values follow the endpoint. The mask comes from an on-policy rollout, and its reconstructed canvas is
counterfactual because provisional rollout values have been replaced by
endpoint tokens from the same teacher response.

\Needspace{7\baselineskip}
\subsection{Native denoising supervision}

Reconstruction now defines the query and its target, and the remaining step is to
score that query through the student's own prediction interface.  The native
input adapter $\operatorname{Input}_a$ maps a reconstruction to a
canvas for bidirectional heads or to clean-prefix/current-block queries for
block heads (\Cref{app:bd3_validity}).  Let
$p^{(a)}_{\theta,i}(\cdot\mid X_i)$ denote the categorical predictor using the
information $X_i$ available when querying position $i$.  CT-OPD minimizes
\begin{equation}
\cL_{\mathrm{CT}}(\theta;\bar\theta)
=\E\!\left[-\frac{1}{|U|\vee1}
\sum_{i\in U}\log p^{(a)}_{\theta,i}(Y_i^\teacher\mid X_i)\right],
\quad X=\operatorname{Input}_a(c,Y^\teacher,\Xi_j).
\label{eq:ct_objective}
\end{equation}
With $U=U(Y^\teacher,M_j)$, the expectation samples training contexts,
endpoints, traces, and a uniform stage $j\in\{0,1,2,3\}$. Empty target sets
contribute zero, and per-example token averaging prevents early states from
receiving more weight solely because they contain more targets.

The fixed trajectory-mask distribution determines which native queries receive
weight, and the categorical score is strictly proper for their conditional
targets.  Its Bayes-optimal predictor is the teacher-token distribution induced by the
reconstructed states and their per-example weights, and excess over that Bayes
risk is a nonnegative weighted conditional KL (\Cref{prop:properness}).  Under the
native reverse-event construction, the same categorical loss equals a
reveal-normalized KL that distinguishes persistence from token prediction, and
for mask-excluding heads this reduces to the absorbing reverse-token KL
(\Cref{app:properness}).  Bidirectional canvases and clean-prefix/current-block
heads can therefore learn through their own conditional factorization.
Because each imperfect supervised token also has a nonzero logit-level
correction when terminal rewards coincide, the objective retains a direct
signal in the zero-contrast groups examined empirically in \Cref{sec:signal_learning}
(\Cref{prop:terminal_starvation}).

\paragraph{How trajectory masks allocate supervision.}
Reconstruction makes each target meaningful in its visible context, and
trajectory-mask selection determines which coherent conditional tasks the
student practices.  At a fixed endpoint, scorer, and active target count,
a trajectory mask and an equal-count uniform mask can select different positions
and leave different endpoint tokens visible around them.  Their risk
difference therefore contains both position-allocation and joint-context
terms, separated exactly in \Cref{eq:trace_allocation}.  With native metadata
matched, a fully masked endpoint has a unique active mask, so position
selection matters as the response becomes partially resolved, where the
student's reveal pattern determines
which conditional queries receive supervision at that stage.

To separate position selection from changes in visible context, we compare
positions within a single reconstructed parent state. Let $U_j$ be unresolved
positions at that state and
$K_j\subset U_j$ the nonempty set still unresolved at the next stage, and write
$R_j=U_j\setminus K_j$ for the positions just revealed.  For any fixed
per-position score $q_i$, write
$\overline q_B=|B|^{-1}\sum_{i\in B}q_i$.  A uniformly selected
$|K_j|$-subset of $U_j$ has expected mean $\overline q_{U_j}$, and
\begin{equation}
\overline q_{K_j}-\overline q_{U_j}
=\frac{|R_j|}{|U_j|}\bigl(\overline q_{K_j}-\overline q_{R_j}\bigr).
\label{eq:main_retained_revealed}
\end{equation}
The left side measures the retained-position excess over an equal-count
uniform subset. We apply this identity to conditional prediction in
\Cref{sec:trajectory_masks}.

\subsection{Online updates across stages and modalities}

CT-OPD refreshes trajectory masks at the start of each cycle to follow the
evolving student. At cycle $k$, one rollout supplies masks at the fully
masked, early, middle, and late denoising stages:
\begin{align}
\bar\theta_k&=\theta_{k,0},\qquad
Z_k\sim\mathbb P_{\bar\theta_k}^{\mathrm{rev}}(\cdot\mid c),
&M_{k,j}&=G_j(Z_k),\nonumber\\
\theta_{k,j+1}&=\operatorname{OptStep}
\bigl(\theta_{k,j},\widehat\cL_{k,j}\bigr),
&&j=0,\ldots,3.
\label{eq:cycle_update}
\end{align}
All four masks stay fixed during the updates, and $\widehat\cL_{k,j}$ averages
stage-$j$ CT losses over the minibatch at parameters $\theta_{k,j}$.
The stages train both
initial response prediction and completion with progressively more endpoint
context. Reusing one trace
amortizes rollout cost over four updates, while refreshing it each cycle adapts
supervision to the student's changing reveal decisions.

The same construction applies to text and image token streams through their
native endpoint encodings and prediction heads. In sparse \model{}, the loss
updates the decoder, router, and selected experts on each reconstructed state.
For unified MMaDA, text and image trajectories supply masks for their respective
endpoints, and the two branches alternate updates to one shared model.
\Cref{app:experiments,app:method} specify all adapters and optimization details.

\section{Experiments}
\label{sec:experiments}

We evaluate CT-OPD across dense, sparse, and unified diffusion VLMs and trace
its gains to three design choices: learning from partial responses, making
visible context agree with the teacher endpoint, and assigning each question
the positions selected by its current student.

\subsection{Experimental setup}

\paragraph{Models and data.}
We study LLaDA-V \citep{you2025lladav}, LaViDa-LLaDA \citep{li2025lavida}, and
our sparse \modelbase{}. Qwen3-VL-32B-Instruct \citep{bai2025qwen3vl} supplies
answer-anchored endpoints for Vision-SR1-47K \citep{li2026visionsr1}, and each
student retokenizes the resulting text under its own response convention.
Vision encoders remain frozen. For unified MMaDA-8B-MixCoT
\citep{yang2025mmada}, training alternates between Qwen3-VL understanding
endpoints and Qwen-Image-2512 \citep{wu2025qwenimage} images encoded by the
frozen native MAGVITv2 tokenizer \citep{yu2024magvitv2}.
\Cref{app:experiments} details data construction, endpoint adapters, training,
and decoding.

\paragraph{Training protocol.}
On all three understanding backbones, CT-OPD uses a 128-token response canvas,
32 reverse steps, and $\{100,75,50,25\}\%$ mask ratios. AdamW
\citep{loshchilov2019adamw} uses global batch 16 and learning rate
$3\times10^{-7}$, and each endpoint trace supplies four sequential updates. All comparison methods use matched training budgets and comparable optimization
settings unless otherwise specified.

\paragraph{Comparisons.}
Base is the initial checkpoint. Endpoint-only and Random share CT-OPD's
teacher endpoints, initialization, optimizer, and update budget. Endpoint-only
trains on fully masked responses, whereas Random masks uniformly selected
positions in teacher responses while matching CT-OPD's stagewise counts and
nested reveal schedule. OPDLM+BPM combines OPDLM \citep{su2026opdlm}
with BPM \citep{wang2026bpm}: the AR teacher scores
causal prefixes of student completions, and BPM aligns its distributions across
tokenizers for forward-KL supervision at masked trajectory positions, retaining
the student's visible rollout tokens (\Cref{app:opdlm_bpm}). Direct Trace-Target
also retains these visible tokens but supervises masked positions with
teacher-endpoint tokens. Trajectory-mask variants vary the trace source
and prompt pairing to examine online refresh and prompt specificity.
MDPO \citep{he2025mdpo}, MaskGRPO \citep{ma2026maskgrpo}, and UniGRPO
\citep{yang2025mmada} are matched-compute diffusion-native RL baselines
(\Cref{app:rl_configuration}). We evaluate the three backbones on nine
understanding benchmarks, and MMaDA on six understanding and three generation
metrics.

\subsection{Results across architectures and modalities}

\begin{table}[t]
\centering
\caption{\textbf{CT-OPD improves dense, sparse, and unified diffusion-VLMs.}
Parentheses give changes from Base (red: gain; blue: decrease); bold marks the
best result per backbone. Endpoint-only uses full masks; Random matches mask
counts and budget; OPDLM+BPM aligns teacher distributions.}
\label{tab:ctopd_unified_main}
\setlength{\tabcolsep}{1.6pt}
\renewcommand{\arraystretch}{1.16}
\arrayrulecolor{tableline}
\scriptsize
\renewcommand{\ctgain}[2]{#1\,\textcolor{tablegain}{\tiny (#2)}}
\renewcommand{\ctbestgain}[2]{{\bfseries #1\,\textcolor{tablegain}{\tiny (#2)}}}
\renewcommand{\ctloss}[2]{#1\,\textcolor{tableloss}{\tiny (#2)}}
\resizebox{\textwidth}{!}{%
\begin{tabular}{@{}lrrrrrrrrrr@{}}
\toprule
\multirow[c]{2}{*}[-0.65ex]{\textbf{Method}}\rule{0pt}{2.7ex}
& \multicolumn{5}{c}{\textbf{Perception and knowledge}}
& \multicolumn{4}{c}{\textbf{Scientific reasoning}}
& \multirow[c]{2}{*}[-0.65ex]{\textbf{\textit{Avg.}}} \\
\cmidrule(lr){2-6}\cmidrule(lr){7-10}
\rule{0pt}{2.5ex}
& MMMU & MMMU-Pro & MMStar & MMBench & SEEDBench
& MathVerse & MathVista & GPQA & GPQA-D & \\
\midrule
\rowcolor{tablegroup}
\multicolumn{11}{l}{\textit{Dense diffusion VLMs}} \\
\rowcolor{tablebase}
LLaDA-V
& 41.78 & 19.02 & 50.07 & 69.45 & 69.70
& 22.72 & 53.60 & 16.96 & 16.67 & 40.00 \\
\rowcolor{tablebase}
\quad + MaskGRPO
& \ctgain{42.33}{+0.55} & \ctgain{25.66}{+6.64}
& \ctgain{52.20}{+2.13} & \ctgain{72.74}{+3.29} & \ctgain{69.79}{+0.09}
& \ctloss{14.72}{$-$8.00} & \ctloss{44.40}{$-$9.20}
& \ctgain{28.57}{+11.61} & \ctgain{29.29}{+12.62}
& \ctgain{42.19}{+2.19} \\
\rowcolor{tablebase}
\quad + MDPO
& \ctloss{41.56}{$-$0.22} & \ctgain{19.25}{+0.23}
& \ctloss{49.27}{$-$0.80} & \ctgain{71.83}{+2.38} & \ctgain{69.87}{+0.17}
& \ctloss{13.45}{$-$9.27} & \ctloss{40.70}{$-$12.90}
& \ctgain{22.77}{+5.81} & \ctgain{24.24}{+7.57}
& \ctloss{39.22}{$-$0.78} \\
\rowcolor{tablebase}
\quad + Endpoint-only
& \ctgain{44.33}{+2.55} & \ctgain{23.87}{+4.85}
& \ctgain{52.93}{+2.86} & \ctgain{75.87}{+6.42} & \ctgain{70.37}{+0.67}
& \ctloss{19.54}{$-$3.18} & \ctloss{48.90}{$-$4.70}
& \ctgain{25.45}{+8.49} & \ctgain{26.77}{+10.10}
& \ctgain{43.12}{+3.12} \\
\rowcolor{tablebase}
\quad + Random
& \ctgain{42.00}{+0.22} & \ctgain{28.55}{+9.54}
& \ctbestgain{56.00}{+5.93} & \ctgain{74.91}{+5.46}
& \ctloss{69.26}{$-$0.44} & \ctgain{23.10}{+0.38}
& \ctgain{56.60}{+3.00} & \ctgain{28.35}{+11.38}
& \ctbestgain{33.84}{+17.17} & \ctgain{45.85}{+5.85} \\
\rowcolor{tablebase}
\quad + OPDLM+BPM
& \ctgain{44.00}{+2.22} & \ctgain{29.60}{+10.58}
& \ctgain{54.40}{+4.33} & \ctgain{80.72}{+11.26}
& \ctgain{70.78}{+1.08} & \ctgain{24.37}{+1.65}
& \ctgain{60.10}{+6.50} & \ctgain{29.46}{+12.50}
& \ctgain{30.81}{+14.14} & \ctgain{47.14}{+7.14} \\
\rowcolor{tableours}
\quad\textbf{+ \method{}}
& \ctbestgain{46.22}{+4.44} & \ctbestgain{33.12}{+14.10}
& \ctgain{55.87}{+5.80} & \ctbestgain{83.16}{+13.71} & \ctbestgain{74.26}{+4.56}
& \ctbestgain{27.79}{+5.08} & \ctbestgain{61.30}{+7.70}
& \ctbestgain{32.59}{+15.62} & \ctbestgain{33.84}{+17.17}
& \cellcolor{tablebest}\ctbestgain{49.79}{+9.80} \\
\addlinespace[1.5pt]
\rowcolor{tablebase}
LaViDa-LLaDA
& 42.00 & 28.79 & 46.27 & 70.50 & 65.98
& 17.51 & 33.00 & 30.13 & 27.78 & 40.22 \\
\rowcolor{tablebase}
\quad + MaskGRPO
& \ctgain{42.67}{+0.67} & \ctloss{24.10}{$-$4.69}
& \ctgain{46.53}{+0.26} & \ctgain{71.40}{+0.90} & \ctloss{65.64}{$-$0.34}
& \ctloss{16.12}{$-$1.39} & \ctgain{34.70}{+1.70}
& \ctloss{29.91}{$-$0.22} & \ctloss{25.76}{$-$2.02}
& \ctloss{39.65}{$-$0.57} \\
\rowcolor{tablebase}
\quad + MDPO
& \ctloss{41.89}{$-$0.11} & \ctloss{24.34}{$-$4.45}
& \ctloss{43.33}{$-$2.94} & \ctgain{70.94}{+0.44} & \ctgain{66.03}{+0.05}
& \ctgain{17.77}{+0.26} & \ctloss{31.00}{$-$2.00}
& \ctbestgain{33.04}{+2.91} & \ctbestgain{32.83}{+5.05}
& \ctloss{40.13}{$-$0.09} \\
\rowcolor{tablebase}
\quad + Endpoint-only
& \ctgain{42.67}{+0.67} & \ctgain{29.36}{+0.57}
& \ctloss{45.60}{$-$0.67} & \ctbestgain{75.51}{+5.01} & \ctgain{66.34}{+0.36}
& \ctbestgain{18.65}{+1.14} & \ctloss{28.30}{$-$4.70}
& \ctloss{29.46}{$-$0.67} & \ctloss{26.77}{$-$1.01}
& \ctgain{40.30}{+0.08} \\
\rowcolor{tablebase}
\quad + Random
& \ctgain{42.11}{+0.11} & \ctgain{28.90}{+0.12}
& \ctgain{46.33}{+0.07} & \ctgain{73.41}{+2.90}
& \ctgain{66.06}{+0.08} & \ctgain{17.89}{+0.38}
& \ctgain{45.70}{+12.70} & \ctloss{29.91}{$-$0.22}
& \ctloss{26.26}{$-$1.52} & \ctgain{41.84}{+1.62} \\
\rowcolor{tablebase}
\quad + OPDLM+BPM
& \ctloss{41.22}{$-$0.78} & \ctgain{29.36}{+0.58}
& \ctloss{46.13}{$-$0.13} & \ctgain{74.57}{+4.07}
& \ctloss{65.49}{$-$0.48} & 17.51
& \ctgain{43.80}{+10.80} & \ctloss{29.02}{$-$1.12}
& \ctloss{24.24}{$-$3.54} & \ctgain{41.26}{+1.04} \\
\rowcolor{tableours}
\quad\textbf{+ \method{}}
& \ctbestgain{43.67}{+1.67} & \ctbestgain{29.83}{+1.04}
& \ctbestgain{47.00}{+0.73} & \ctgain{75.37}{+4.87} & \ctbestgain{66.53}{+0.55}
& \ctgain{18.15}{+0.63} & \ctbestgain{47.60}{+14.60}
& 30.13 & 27.78
& \cellcolor{tablebest}\ctbestgain{42.90}{+2.68} \\
\midrule
\rowcolor{tablegroup}
\multicolumn{11}{l}{\textit{Sparse diffusion VLM}} \\
\rowcolor{tablebase}
\modelbase{}
& 32.67 & \textbf{21.85} & 40.00 & 71.30 & 61.34
& 18.65 & 34.60 & 25.45 & 25.25 & 36.79 \\
\rowcolor{tablebase}
\quad + MaskGRPO
& \ctgain{33.56}{+0.89} & \ctloss{20.75}{$-$1.10}
& \ctgain{41.13}{+1.13} & \ctgain{72.17}{+0.87} & \ctgain{64.18}{+2.84}
& \ctloss{17.89}{$-$0.76} & \ctloss{32.00}{$-$2.60}
& \ctloss{23.44}{$-$2.01} & \ctloss{18.18}{$-$7.07}
& \ctloss{35.92}{$-$0.87} \\
\rowcolor{tablebase}
\quad + MDPO
& \ctgain{37.78}{+5.11} & \ctloss{21.62}{$-$0.23}
& \ctloss{38.93}{$-$1.07} & \ctbestgain{76.60}{+5.30} & \ctgain{62.21}{+0.87}
& \ctloss{14.97}{$-$3.68} & \ctloss{34.20}{$-$0.40}
& \ctgain{26.79}{+1.34} & \ctgain{30.30}{+5.05}
& \ctgain{38.16}{+1.37} \\
\rowcolor{tablebase}
\quad + Endpoint-only
& \ctgain{35.33}{+2.66} & \ctloss{21.62}{$-$0.23}
& \ctloss{38.07}{$-$1.93} & \ctloss{58.44}{$-$12.86} & \ctloss{55.99}{$-$5.35}
& \ctloss{18.53}{$-$0.12} & \ctgain{39.30}{+4.70}
& \ctgain{26.79}{+1.34} & \ctgain{26.77}{+1.52}
& \ctloss{35.65}{$-$1.14} \\
\rowcolor{tablebase}
\quad + Random
& \ctgain{40.89}{+8.22} & \ctloss{19.31}{$-$2.54}
& \ctloss{38.07}{$-$1.93} & \ctloss{59.13}{$-$12.18}
& \ctloss{58.80}{$-$2.54} & \ctgain{18.91}{+0.25}
& \ctgain{44.50}{+9.90} & \ctbestgain{30.13}{+4.69}
& \ctgain{30.30}{+5.05} & \ctgain{37.78}{+0.99} \\
\rowcolor{tablebase}
\quad + OPDLM+BPM
& \ctgain{39.00}{+6.33} & \ctloss{20.92}{$-$0.92}
& \ctloss{39.20}{$-$0.80} & \ctloss{69.09}{$-$2.22}
& \ctgain{63.79}{+2.45} & \ctgain{20.18}{+1.52}
& \ctgain{42.10}{+7.50} & \ctgain{28.57}{+3.12}
& \ctgain{28.28}{+3.03} & \ctgain{39.01}{+2.22} \\
\rowcolor{tableours}
\quad\textbf{+ \method{}}
& \ctbestgain{41.11}{+8.44} & \ctloss{21.39}{$-$0.46}
& \ctbestgain{42.73}{+2.73} & \ctgain{75.85}{+4.55} & \ctbestgain{65.94}{+4.60}
& \ctbestgain{20.69}{+2.03} & \ctbestgain{44.70}{+10.10}
& \ctbestgain{30.13}{+4.69} & \ctbestgain{30.81}{+5.56}
& \cellcolor{tablebest}\ctbestgain{41.48}{+4.69} \\
\midrule
\rowcolor{tablegroup}
\multicolumn{11}{l}{\textit{Unified understanding-and-generation diffusion VLM}} \\
\multirow[c]{2}{*}[-0.65ex]{\textbf{Method}}
& \multicolumn{7}{c}{\textbf{Multimodal understanding}}
& \multicolumn{3}{c}{\textbf{Image generation}} \\
\cmidrule(lr){2-8}\cmidrule(lr){9-11}
& POPE & MME & GQA & MMMU & MMBench & SEEDBench
& \textbf{\textit{Avg.}} & GenEval & CLIP & ImageReward \\
\midrule
\rowcolor{tablebase}
MMaDA-8B-MixCoT
& 68.99 & 62.85 & 48.68
& 31.43 & 44.07 & 56.13
& 52.03 & 55.45 & 28.43
& 0.585 \\
\rowcolor{tablebase}
\quad + UniGRPO
& \ctloss{65.68}{$-$3.31}
& \ctgain{64.81}{+1.96}
& \ctloss{48.39}{$-$0.29}
& \ctloss{30.33}{$-$1.10}
& \ctgain{44.16}{+0.09}
& \ctgain{56.55}{+0.42}
& \ctloss{51.65}{$-$0.38}
& \ctgain{59.90}{+4.44}
& \ctgain{28.81}{+0.38}
& \ctgain{0.752}{+0.167} \\
\rowcolor{tablebase}
\quad + Endpoint-only
& \ctgain{78.71}{+9.72}
& \ctbestgain{65.07}{+2.22}
& \ctgain{48.73}{+0.05}
& \ctgain{32.57}{+1.14}
& \ctgain{53.94}{+9.87}
& \ctgain{59.28}{+3.15}
& \ctgain{56.38}{+4.35}
& \ctgain{60.70}{+5.25}
& \ctgain{28.99}{+0.56}
& \ctgain{0.732}{+0.147} \\
\rowcolor{tablebase}
\quad + Random
& \ctgain{79.50}{+10.51} & \ctgain{63.14}{+0.29}
& \ctgain{55.11}{+6.43} & \ctgain{32.95}{+1.52}
& \ctgain{59.30}{+15.22} & \ctgain{59.98}{+3.85}
& \ctgain{58.33}{+6.31} & \ctgain{60.06}{+4.60}
& \ctgain{29.01}{+0.58} & \ctgain{0.765}{+0.180} \\
\rowcolor{tableours}
\quad\textbf{+ \method{}}
& \ctbestgain{79.72}{+10.73}
& \ctgain{63.48}{+0.63}
& \ctbestgain{56.38}{+7.70}
& \ctbestgain{34.19}{+2.76}
& \ctbestgain{62.14}{+18.07}
& \ctbestgain{61.13}{+5.00}
& \cellcolor{tablebest}\ctbestgain{59.51}{+7.48}
& \ctbestgain{61.69}{+6.23}
& \ctbestgain{29.16}{+0.73}
& \ctbestgain{0.774}{+0.189} \\
\bottomrule
\end{tabular}%
}
\arrayrulecolor{black}
\end{table}

\Cref{tab:ctopd_unified_main} shows that CT-OPD achieves the highest
nine-benchmark average on all three understanding backbones. It outperforms
OPDLM+BPM on every benchmark, raising the averages by 2.66, 1.63, and 2.47
points on LLaDA-V, LaViDa-LLaDA, and sparse \modelbase{}, respectively.

Random improves on Endpoint-only by learning to complete partially visible
teacher responses. With teacher content and mask counts matched, CT-OPD
surpasses Random by 3.95, 1.05, and 3.70 points on average, demonstrating the
benefit of student-selected positions. Relative to Base, CT-OPD improves all
nine LLaDA-V tasks, gaining 9.80 points on average, while LaViDa-LLaDA's
14.60-point gain on MathVista highlights improved visual mathematical reasoning.

\Needspace{5\baselineskip}
On unified MMaDA, CT-OPD achieves the highest understanding average and the
best results on all three generation metrics, improving over Base by 7.48
points in understanding and 6.23 on GenEval. It surpasses Random by 1.17 and
1.63 points, respectively, and reaches 29.16 CLIP and 0.774 ImageReward,
compared with Random's 29.01 and 0.765. UniGRPO improves generation but reduces
understanding by 0.38 points, whereas CT-OPD benefits both modalities through
the same reconstruction: trajectory masks specify unresolved positions, and
teacher endpoints supply coherent context and targets in both text and image
token spaces. \Cref{tab:mmada_joint_generation,fig:mmada_generation_qualitative}
give category-level results and matched generations across compositional
image-generation settings.

\subsection{Learning from partial responses}
\label{sec:signal_learning}

To understand these gains, we compare CT-OPD and Endpoint-only on the same
reconstructed states, testing whether partial-state training improves the
student's use of progressively revealed context. We examine predictions across
denoising stages and then focus on prompts with tied
terminal rewards, where outcome feedback provides no contrast.

\begin{table}[t]
\centering
\caption{\textbf{Learning under zero terminal reward contrast.} Coverage spans
four stages on 311 zero-contrast prompts; prediction uses their 292 all-wrong
prompts. Parentheses give CT-OPD differences with paired 95\% CIs below;
NLL is in nats/token and accuracy in percentage points.}
\label{tab:zero_contrast_learning}
\vspace{1pt}
\arrayrulecolor{tableline}
\scriptsize
\setlength{\tabcolsep}{3.2pt}
\renewcommand{\arraystretch}{1.08}
\begin{tabularx}{\textwidth}{@{}>{\raggedright\arraybackslash}p{0.285\textwidth}*{4}{>{\centering\arraybackslash}X}@{}}
\toprule
\rowcolor{tablegroup}
\multicolumn{5}{l}{\textit{Terminal contrast and zero-contrast CT coverage}} \\
\rowcolor{tablebase}
Frozen terminal groups & \multicolumn{2}{c}{Zero contrast: \textbf{60.74\%}} & \multicolumn{2}{c}{With contrast: 39.26\%} \\
\rowcolor{tableours}
CT coverage (zero contrast) & \multicolumn{4}{c}{\makebox[0.66\textwidth][c]{\makebox[0.22\textwidth][c]{Prompt: \textbf{100.00\%}}\makebox[0.22\textwidth][c]{State: \textbf{91.48\%}}\makebox[0.22\textwidth][c]{Token: \textbf{60.79\%}}}} \\
\midrule
\rowcolor{tablegroup}
\multicolumn{3}{l}{\textit{Common-state learning on 292 all-wrong groups}} & \multicolumn{2}{r}{\textcolor{tablemuted}{876 states; 35,935 scored tokens}} \\
Method & \multicolumn{2}{c}{NLL $\downarrow$} & \multicolumn{2}{c}{Acc. $\uparrow$} \\
\midrule
\rowcolor{tablebase}
Base (frozen) & 4.6073 & \textcolor{tablegain}{\tiny ($-$1.2412)} & 20.56 & \textcolor{tablegain}{\tiny (+6.87)} \\
\rowcolor{tablebase}
 & \multicolumn{2}{c}{\textcolor{tablemuted}{\tiny 95\% CI [$-$1.3331, $-$1.1545]}} & \multicolumn{2}{c}{\textcolor{tablemuted}{\tiny 95\% CI [6.11, 7.66]}} \\
\rowcolor{tablebase}
Direct Trace-Target (raw) & 4.1878 & \textcolor{tablegain}{\tiny ($-$0.8217)} & 13.68 & \textcolor{tablegain}{\tiny (+13.75)} \\
\rowcolor{tablebase}
 & \multicolumn{2}{c}{\textcolor{tablemuted}{\tiny 95\% CI [$-$0.8775, $-$0.7658]}} & \multicolumn{2}{c}{\textcolor{tablemuted}{\tiny 95\% CI [12.55, 15.01]}} \\
\rowcolor{tablebase}
Endpoint-only ($4\times100\%$) & 3.5736 & \textcolor{tablegain}{\tiny ($-$0.2075)} & 24.15 & \textcolor{tablegain}{\tiny (+3.28)} \\
\rowcolor{tablebase}
 & \multicolumn{2}{c}{\textcolor{tablemuted}{\tiny 95\% CI [$-$0.2320, $-$0.1833]}} & \multicolumn{2}{c}{\textcolor{tablemuted}{\tiny 95\% CI [2.75, 3.84]}} \\
\rowcolor{tablebase}
Random & 3.5202 & \textcolor{tablegain}{\tiny ($-$0.1541)} & 26.12 & \textcolor{tablegain}{\tiny (+1.31)} \\
\rowcolor{tablebase}
 & \multicolumn{2}{c}{\textcolor{tablemuted}{\tiny 95\% CI [$-$0.1802, $-$0.1290]}} & \multicolumn{2}{c}{\textcolor{tablemuted}{\tiny 95\% CI [0.91, 1.71]}} \\
\rowcolor{tableours}
\textbf{CT-OPD (student trace)} & \multicolumn{2}{c}{\textbf{3.3661}} & \multicolumn{2}{c}{\textbf{27.43}} \\
\midrule
\rowcolor{tablegroup}
\multicolumn{3}{l}{\textit{Partial-state learning across the reverse process}} & \multicolumn{2}{r}{\textcolor{tablemuted}{1,168 states; 60,864 scored tokens}} \\
Mask & End NLL $\downarrow$ & CT NLL $\downarrow$ & End acc. $\uparrow$ & CT acc. $\uparrow$ \\
\midrule
\rowcolor{tablebase}
100\% & \textbf{4.132} & 4.234\,\textcolor{tableloss}{\tiny (+0.1016)} & \textbf{15.88} & 14.94\,\textcolor{tableloss}{\tiny ($-$0.94)} \\[-0.35ex]
\rowcolor{tablebase}
 & \multicolumn{2}{r}{\textcolor{tablemuted}{\tiny 95\% CI [0.0869, 0.1167]}} & \multicolumn{2}{r}{\textcolor{tablemuted}{\tiny 95\% CI [$-$1.26, $-$0.63]}} \\[0.1ex]
\rowcolor{tablebase}
75\% & 3.959 & \textbf{3.819}\,\textcolor{tablegain}{\tiny ($-$0.1406)} & 17.52 & \textbf{19.26}\,\textcolor{tablegain}{\tiny (+1.75)} \\[-0.35ex]
\rowcolor{tablebase}
 & \multicolumn{2}{r}{\textcolor{tablemuted}{\tiny 95\% CI [$-$0.1667, $-$0.1147]}} & \multicolumn{2}{r}{\textcolor{tablemuted}{\tiny 95\% CI [1.19, 2.33]}} \\[0.1ex]
\rowcolor{tablebase}
50\% & 3.464 & \textbf{3.226}\,\textcolor{tablegain}{\tiny ($-$0.2386)} & 25.60 & \textbf{29.38}\,\textcolor{tablegain}{\tiny (+3.78)} \\[-0.35ex]
\rowcolor{tablebase}
 & \multicolumn{2}{r}{\textcolor{tablemuted}{\tiny 95\% CI [$-$0.2759, $-$0.2018]}} & \multicolumn{2}{r}{\textcolor{tablemuted}{\tiny 95\% CI [2.79, 4.81]}} \\[0.1ex]
\rowcolor{tablebase}
25\% & 2.531 & \textbf{2.169}\,\textcolor{tablegain}{\tiny ($-$0.3622)} & 42.95 & \textbf{50.27}\,\textcolor{tablegain}{\tiny (+7.31)} \\[-0.35ex]
\rowcolor{tablebase}
 & \multicolumn{2}{r}{\textcolor{tablemuted}{\tiny 95\% CI [$-$0.4239, $-$0.3014]}} & \multicolumn{2}{r}{\textcolor{tablemuted}{\tiny 95\% CI [5.82, 8.93]}} \\[0.1ex]
\bottomrule
\end{tabularx}
\arrayrulecolor{black}
\end{table}

The stagewise probe shows where the additional training tasks pay off.
Endpoint-only predicts best from a full mask, where all mask choices collapse
to the same prediction task, and CT-OPD is stronger at every measured partial
state, with its NLL advantage growing from 0.14 at 75\% mask to 0.36 at 25\%.
The normalized loss gives each remaining target greater weight as the active
target set shrinks.
Both checkpoints
are evaluated on the same reconstructed states, so this progression reflects
their learned ability to use the context progressively exposed by decoding in
each partial-state query.

Terminal group-relative updates lose contrast when completed rewards tie.
Four frozen Base completions split 512 prompts into all-wrong, all-correct,
and mixed groups, and 60.74\% have identical correctness rewards.  Endpoint
reconstruction supplies token targets for every zero-contrast prompt, covering
91.48\% of their common evaluation states and 60.79\% of endpoint-token
positions across stages.  On the 292 all-wrong prompts, CT-OPD achieves the
lowest teacher-target NLL and highest token accuracy, with statistically
significant paired improvements over Endpoint-only and Random
(\Cref{tab:zero_contrast_learning}).  These results demonstrate effective
conditional learning on prompts where sampled terminal rewards provide no
within-group contrast.

\subsection{Coherent teacher context}

Position transfer remains ambiguous if the raw trace retains the
student's provisional response.  Direct Trace-Target shares CT-OPD's endpoints,
online trajectory-mask selection protocol, and exposure but scores teacher
tokens beside the trace's provisional visible values.  When one differs from
the endpoint, context and target describe different responses.  CT-OPD instead
fills visible positions from the same endpoint, lifting LLaDA-V's nine-task average from 45.99 to
49.79.  On common reconstructed states, CT-OPD has lower NLL
and higher token accuracy, with an advantage over Direct Trace-Target that
grows as more endpoint content becomes visible
(\Cref{app:conditional_prediction_probe}).  The trace thus selects the
queries, while the endpoint supplies targets together with the visible
response context.

\subsection{Student-selected, prompt-matched trajectory masks}
\label{sec:trajectory_masks}

After coherent reconstruction, the remaining design choice is which positions
should stay unresolved. Even at the
same mask size, different position sets expose different targets and visible
context. Each visible endpoint token supplies conditioning evidence for
predicting the remaining teacher targets. We therefore vary the mask source (current student or frozen Base),
prompt assignment (matched or cross-prompt), and position selection
(trajectory-based or count-matched random), while holding endpoint content,
stage exposure, initialization, and optimization fixed.

\begin{table}[t]
\centering
\caption{\textbf{Online, prompt-matched trajectory masks give the strongest transfer.}
The controls vary mask source and prompt pairing; Random preserves
mask counts while randomizing positions.  The lower block measures positions
retained by the trajectory on the same reconstructed parent.}
\label{tab:trace_ablation}
\setlength{\tabcolsep}{3.15pt}
\renewcommand{\arraystretch}{1.07}
\arrayrulecolor{tableline}
\scriptsize
\resizebox{\textwidth}{!}{%
\begin{tabular}{@{}lrrrrrrrrrr@{}}
\toprule
\textbf{Mask setting}
& \textbf{MMMU} & \textbf{MMMU-Pro} & \textbf{MMStar}
& \textbf{MMBench} & \textbf{SEEDBench} & \textbf{MathVerse}
& \textbf{MathVista} & \textbf{GPQA} & \textbf{GPQA-D}
& \textbf{\textit{Avg.}} \\
\midrule
\rowcolor{tableours}
Online matched (\method{})
& \textbf{46.22} & \textbf{33.12} & 55.87 & \textbf{83.16}
& \textbf{74.26} & \textbf{27.79} & \textbf{61.30}
& \textbf{32.59} & \textbf{33.84} & \cellcolor{tablebest}\textbf{49.79} \\
\rowcolor{tablebase}
Online shuffled
& 44.00 & 31.97 & 53.93 & 79.69 & 72.11
& 25.89 & 58.90 & 31.03 & 29.29 & 47.42 \\
\rowcolor{tablebase}
Frozen-Base matched
& 45.33 & 32.60 & 55.27 & 82.82 & 73.57
& 26.02 & 59.20 & 31.70 & 30.81 & 48.59 \\
\rowcolor{tablebase}
Frozen-Base shuffled
& 42.67 & 31.79 & 53.40 & 76.90 & 71.24
& 24.75 & 58.40 & 29.91 & 29.80 & 46.54 \\
\addlinespace[1pt]
\rowcolor{tablebase}
Random
& 42.00 & 28.55 & \textbf{56.00} & 74.91 & 69.26
& 23.10 & 56.60 & 28.35 & \textbf{33.84} & 45.85 \\
\bottomrule
\end{tabular}%
}
\par\nointerlineskip
\vspace{1.5pt}
\begin{tabularx}{\textwidth}{@{}>{\raggedright\arraybackslash}p{0.27\textwidth}X@{}}
\toprule
\rowcolor{tablegroup}
\multicolumn{2}{l}{\textit{Position-allocation measurements on the frozen Base panel}} \\
\midrule
\rowcolor{tableours}
Position-allocation margins
& $\mathcal{T}=\mathbf{1.1123}$
  \textcolor{tablemuted}{\tiny (95\% CI [0.9325, 1.3053])};
  $\mathfrak M(c)=\mathbf{0.1050}$
  \textcolor{tablemuted}{\tiny (95\% CI [0.0907, 0.1187])} \\
\multicolumn{2}{l}{\textcolor{tablemuted}{\tiny Token-gradient margin is mean unnormalized categorical-logit $\ell_1$ magnitude; paired-prompt 95\% CIs.}} \\
\bottomrule
\end{tabularx}
\end{table}

CT-OPD achieves the highest nine-task average among the five mask settings
(\Cref{tab:trace_ablation}). Prompt matching improves the average by 2.37
points with online traces and 2.05 points with frozen traces, while online
refresh adds 1.20 points under matched pairing and 0.88 points under shuffled
pairing. Both gains persist under either setting of the other factor,
supporting online refresh and prompt matching as complementary design choices.
The online matched setting also leads all nine tasks among the four
trace-source and pairing variants. Random trails it by 3.95 points
despite identical active counts. This advantage over Random extends across
backbones and modalities (\Cref{tab:ctopd_unified_main}).

Because the teacher endpoint replaces the provisional context in which the
student selected its unresolved positions, we test whether those positions
remain difficult in the reconstructed states used for training. Before
training, we fix the Base scorer and
reconstructed parent state, then compare positions retained at the next stage
with an equally sized uniform subset of the parent's unresolved positions.
Applying \Cref{eq:main_retained_revealed} to teacher-target NLL isolates
position selection within this shared context. Retained positions have
$1.1123$ higher teacher-target NLL (95\% CI $[0.9325,1.3053]$) and a larger
mean categorical-logit correction than count-matched random positions
(\Cref{app:difficulty_transport}). Together with the controlled training comparisons, this result shows that
trajectory-based position selection directs supervision toward harder
conditional predictions, while online refresh and prompt matching preserve
this advantage as the model evolves.

\section{Conclusion}
\label{sec:conclusion}

Diffusion VLMs generate through a sequence of partially resolved states, making
intermediate-state supervision central to effective post-training. We
introduced CT-OPD, which separates where supervision is applied from what
response is learned. The current model's trajectory determines the unresolved
positions, while a teacher endpoint supplies both visible context and prediction
targets. By discarding provisional rollout values and reconstructing each
partial state from one coherent endpoint, CT-OPD provides native token-level
supervision along the model's evolving reverse process, including cases where
terminal rewards provide no contrast. Across dense and sparse diffusion VLMs,
CT-OPD consistently improves multimodal understanding and reasoning, while
unified MMaDA shows that the same principle also transfers to image generation.
Controlled studies further attribute these gains to partial-state training,
endpoint-consistent reconstruction, and current, prompt-matched trajectory
masks, with student-retained positions exhibiting higher teacher-target error
than count-matched random positions after reconstruction. Together, these
results support a simple principle for diffusion VLM post-training in which
the model's trajectory determines where to learn, while the teacher endpoint
determines what to learn.




\bibliography{references}

\begin{thebibliography}{33}
\providecommand{\natexlab}[1]{#1}
\providecommand{\url}[1]{\texttt{#1}}
\expandafter\ifx\csname urlstyle\endcsname\relax
  \providecommand{\doi}[1]{doi: #1}\else
  \providecommand{\doi}{doi: \begingroup \urlstyle{rm}\Url}\fi

\bibitem[Agarwal et~al.(2024)Agarwal, Vieillard, Zhou, Stanczyk, Garea, Geist, and Bachem]{agarwal2023opd}
Rishabh Agarwal, Nino Vieillard, Yongchao Zhou, Piotr Stanczyk, Sabela~Ramos Garea, Matthieu Geist, and Olivier Bachem.
\newblock On-policy distillation of language models: Learning from self-generated mistakes.
\newblock In \emph{The Twelfth International Conference on Learning Representations, {ICLR} 2024, Vienna, Austria, May 7-11, 2024}. OpenReview.net, 2024.
\newblock URL \url{https://openreview.net/forum?id=3zKtaqxLhW}.

\bibitem[Ahmadi et~al.(2026)Ahmadi, Parthasarathi, and Cui]{ahmadi2026trafl}
Saba Ahmadi, Prasanna Parthasarathi, and Yufei Cui.
\newblock Beyond mode-seeking {RL:} trajectory-balance post-training for diffusion language models.
\newblock \emph{CoRR}, abs/2605.13935, 2026.
\newblock \doi{10.48550/ARXIV.2605.13935}.
\newblock URL \url{https://doi.org/10.48550/arXiv.2605.13935}.

\bibitem[Bai et~al.(2025)Bai, Cai, Chen, Chen, Chen, Cheng, Deng, Ding, Gao, Ge, Ge, Guo, Huang, Huang, Huang, Hui, Jiang, Li, Li, Li, Li, Lin, Lin, Liu, Liu, Liu, Liu, Liu, Liu, Lu, Luo, Lv, Men, Meng, Ren, Ren, Song, Sun, Tang, Tu, Wan, Wang, Wang, Wang, Wang, Xie, Xu, Xu, Xu, Yang, Yang, Yang, Yang, Yu, Zhang, Zhang, Zhang, Zheng, Zhong, Zhou, Zhou, Zhou, Zhu, and Zhu]{bai2025qwen3vl}
Shuai Bai, Yuxuan Cai, Ruizhe Chen, Keqin Chen, Xionghui Chen, Zesen Cheng, Lianghao Deng, Wei Ding, Chang Gao, Chunjiang Ge, Wenbin Ge, Zhifang Guo, Qidong Huang, Jie Huang, Fei Huang, Binyuan Hui, Shutong Jiang, Zhaohai Li, Mingsheng Li, Mei Li, Kaixin Li, Zicheng Lin, Junyang Lin, Xuejing Liu, Jiawei Liu, Chenglong Liu, Yang Liu, Dayiheng Liu, Shixuan Liu, Dunjie Lu, Ruilin Luo, Chenxu Lv, Rui Men, Lingchen Meng, Xuancheng Ren, Xingzhang Ren, Sibo Song, Yuchong Sun, Jun Tang, Jianhong Tu, Jianqiang Wan, Peng Wang, Pengfei Wang, Qiuyue Wang, Yuxuan Wang, Tianbao Xie, Yiheng Xu, Haiyang Xu, Jin Xu, Zhibo Yang, Mingkun Yang, Jianxin Yang, An~Yang, Bowen Yu, Fei Zhang, Hang Zhang, Xi~Zhang, Bo~Zheng, Humen Zhong, Jingren Zhou, Fan Zhou, Jing Zhou, Yuanzhi Zhu, and Ke~Zhu.
\newblock Qwen3-vl technical report, 2025.
\newblock URL \url{https://arxiv.org/abs/2511.21631}.

\bibitem[Boizard et~al.(2025)Boizard, Haddad, Hudelot, and Colombo]{boizard2024uld}
Nicolas Boizard, Kevin~El Haddad, Céline Hudelot, and Pierre Colombo.
\newblock Towards cross-tokenizer distillation: the universal logit distillation loss for llms, 2025.
\newblock URL \url{https://arxiv.org/abs/2402.12030}.

\bibitem[Chen et~al.(2026)Chen, Xia, Tu, Shi, Zhang, Yao, Yuan, and Zhu]{chen2026bard}
Baoyou Chen, Hanchen Xia, Peng Tu, Haojun Shi, Liwei Zhang, Yuxuan Yao, Weihao Yuan, and Siyu Zhu.
\newblock Bard: Bridging autoregressive and diffusion vision-language models via highly efficient progressive block merging and stage-wise distillation, 2026.
\newblock URL \url{https://arxiv.org/abs/2604.16514}.

\bibitem[Dat et~al.(2026)Dat, Li, and Wang]{dat2026dopsd}
Phuong~Tuan Dat, Qi~Li, and Xinchao Wang.
\newblock dopsd: On-policy self-distillation for diffusion language models.
\newblock \emph{CoRR}, abs/2607.04428, 2026.
\newblock \doi{10.48550/ARXIV.2607.04428}.
\newblock URL \url{https://doi.org/10.48550/arXiv.2607.04428}.

\bibitem[Fedus et~al.(2022)Fedus, Zoph, and Shazeer]{fedus2021switch}
William Fedus, Barret Zoph, and Noam Shazeer.
\newblock Switch transformers: Scaling to trillion parameter models with simple and efficient sparsity.
\newblock \emph{J. Mach. Learn. Res.}, 23:\penalty0 120:1--120:39, 2022.
\newblock URL \url{https://jmlr.org/papers/v23/21-0998.html}.

\bibitem[He et~al.(2025)He, Renz, Cao, and Geiger]{he2025mdpo}
Haoyu He, Katrin Renz, Yong Cao, and Andreas Geiger.
\newblock {MDPO:} overcoming the training-inference divide of masked diffusion language models.
\newblock \emph{CoRR}, abs/2508.13148, 2025.
\newblock \doi{10.48550/ARXIV.2508.13148}.
\newblock URL \url{https://doi.org/10.48550/arXiv.2508.13148}.

\bibitem[Kim \& Rush(2016)Kim and Rush]{kim2016sequencekd}
Yoon Kim and Alexander~M. Rush.
\newblock Sequence-level knowledge distillation, 2016.
\newblock URL \url{https://arxiv.org/abs/1606.07947}.

\bibitem[Li et~al.(2025)Li, Kallidromitis, Bansal, Gokul, Kato, Kozuka, Kuen, Lin, Chang, and Grover]{li2025lavida}
Shufan Li, Konstantinos Kallidromitis, Hritik Bansal, Akash Gokul, Yusuke Kato, Kazuki Kozuka, Jason Kuen, Zhe Lin, Kai{-}Wei Chang, and Aditya Grover.
\newblock Lavida: {A} large diffusion language model for multimodal understanding.
\newblock In Danielle Belgrave, Cheng Zhang, Laura~N. Montoya, Hsuan{-}Tien Lin, Razvan Pascanu, Piotr Koniusz, Marzyeh Ghassemi, Nancy Chen, Iv{\'{a}}n Vladimir~Meza Ru{\'{\i}}z, and Arturo Loaiza{-}Bonilla (eds.), \emph{Advances in Neural Information Processing Systems 38: Annual Conference on Neural Information Processing Systems 2025, NeurIPS 2025, San Diego, CA, USA, December 2-7, 2025 / Mexico City, Mexico, November 30 - December 5, 2025}, 2025.
\newblock URL \url{http://papers.nips.cc/paper\_files/paper/2025/hash/975affbe7b5f3b55fba62247b6877b1c-Abstract-Conference.html}.

\bibitem[Li et~al.(2026{\natexlab{a}})Li, Zhu, Gu, Liu, Lin, Chen, Tao, Grover, and Kuen]{li2026lavidar1}
Shufan Li, Yuchen Zhu, Jiuxiang Gu, Kangning Liu, Zhe Lin, Yongxin Chen, Molei Tao, Aditya Grover, and Jason Kuen.
\newblock Lavida-r1: Advancing reasoning for unified multimodal diffusion language models.
\newblock \emph{CoRR}, abs/2602.14147, 2026{\natexlab{a}}.
\newblock \doi{10.48550/ARXIV.2602.14147}.
\newblock URL \url{https://doi.org/10.48550/arXiv.2602.14147}.

\bibitem[Li et~al.(2026{\natexlab{b}})Li, Yu, Huang, Liang, Liu, Liu, Che, Yu, Boyd-Graber, Mi, and Yu]{li2026visionsr1}
Zongxia Li, Wenhao Yu, Chengsong Huang, Zhenwen Liang, Rui Liu, Fuxiao Liu, Jingxi Che, Dian Yu, Jordan Boyd-Graber, Haitao Mi, and Dong Yu.
\newblock Self-rewarding vision-language model via reasoning decomposition, 2026{\natexlab{b}}.
\newblock URL \url{https://arxiv.org/abs/2508.19652}.

\bibitem[Lin et~al.(2026)Lin, Tang, Ye, Huang, Zhang, Pang, Jin, Ning, Luo, and Yuan]{lin2024moellava}
Bin Lin, Zhenyu Tang, Yang Ye, Jinfa Huang, Junwu Zhang, Yatian Pang, Peng Jin, Munan Ning, Jiebo Luo, and Li~Yuan.
\newblock Moe-llava: Mixture of experts for large vision-language models.
\newblock \emph{{IEEE} Trans. Multim.}, 28:\penalty0 4408--4419, 2026.
\newblock \doi{10.1109/TMM.2026.3654458}.
\newblock URL \url{https://doi.org/10.1109/TMM.2026.3654458}.

\bibitem[Loshchilov \& Hutter(2019)Loshchilov and Hutter]{loshchilov2019adamw}
Ilya Loshchilov and Frank Hutter.
\newblock Decoupled weight decay regularization.
\newblock In \emph{7th International Conference on Learning Representations, {ICLR} 2019, New Orleans, LA, USA, May 6-9, 2019}. OpenReview.net, 2019.
\newblock URL \url{https://openreview.net/forum?id=Bkg6RiCqY7}.

\bibitem[Lu et~al.(2026)Lu, Zhang, Guo, Zhang, Pang, Liu, Li, Gu, Che, Guo, and Guo]{lu2026optd}
Xiaocheng Lu, Hualei Zhang, Shuhan Guo, Jie Zhang, Xiaoyi Pang, Jian Liu, Haoxi Li, Bohai Gu, Haoxuan Che, Jingcai Guo, and Song Guo.
\newblock {OPTD:} on-policy transition distillation with consistency-guided adaptive compression for few-step diffusion language models.
\newblock \emph{CoRR}, abs/2608.02942, 2026.
\newblock \doi{10.48550/ARXIV.2608.02942}.
\newblock URL \url{https://doi.org/10.48550/arXiv.2608.02942}.

\bibitem[Ma et~al.(2025)Ma, Zhang, Wang, and Ye]{ma2026maskgrpo}
Tianren Ma, Mu~Zhang, Yibing Wang, and Qixiang Ye.
\newblock Consolidating reinforcement learning for multimodal discrete diffusion models.
\newblock \emph{CoRR}, abs/2510.02880, 2025.
\newblock \doi{10.48550/ARXIV.2510.02880}.
\newblock URL \url{https://doi.org/10.48550/arXiv.2510.02880}.

\bibitem[Minixhofer et~al.(2025)Minixhofer, Vuli\'{c}, and Ponti]{minixhofer2025alm}
Benjamin Minixhofer, Ivan Vuli\'{c}, and Edoardo~Maria Ponti.
\newblock Universal cross-tokenizer distillation via approximate likelihood matching.
\newblock In D.~Belgrave, C.~Zhang, H.~Lin, R.~Pascanu, P.~Koniusz, M.~Ghassemi, and N.~Chen (eds.), \emph{Advances in Neural Information Processing Systems}, volume 38, Main Conference, pp.\  79297--79326. Curran Associates, Inc., 2025.
\newblock \doi{10.52202/085713-2653}.
\newblock URL \url{https://proceedings.neurips.cc/paper_files/paper/2025/file/720f9f5dc751eb56952ae4fee2398f73-Paper-Conference.pdf}.

\bibitem[Nie et~al.(2025)Nie, Zhu, You, Zhang, Ou, Hu, Zhou, Lin, Wen, and Li]{nie2025llada}
Shen Nie, Fengqi Zhu, Zebin You, Xiaolu Zhang, Jingyang Ou, Jun Hu, Jun Zhou, Yankai Lin, Ji{-}Rong Wen, and Chongxuan Li.
\newblock Large language diffusion models.
\newblock \emph{CoRR}, abs/2502.09992, 2025.
\newblock \doi{10.48550/ARXIV.2502.09992}.
\newblock URL \url{https://doi.org/10.48550/arXiv.2502.09992}.

\bibitem[Oba et~al.(2026)Oba, Furuta, and Okazaki]{oba2026dispo}
Daisuke Oba, Hiroki Furuta, and Naoaki Okazaki.
\newblock Diffusion-state policy optimization for masked diffusion language models.
\newblock \emph{CoRR}, abs/2602.06462, 2026.
\newblock \doi{10.48550/ARXIV.2602.06462}.
\newblock URL \url{https://doi.org/10.48550/arXiv.2602.06462}.

\bibitem[Qian et~al.(2026)Qian, Su, Hu, Zhang, Deng, Zhao, and Zhang]{qian2026d3llm}
Yu{-}Yang Qian, Junda Su, Lanxiang Hu, Peiyuan Zhang, Zhijie Deng, Peng Zhao, and Hao Zhang.
\newblock d3llm: Ultra-fast diffusion {LLM} using pseudo-trajectory distillation.
\newblock \emph{CoRR}, abs/2601.07568, 2026.
\newblock \doi{10.48550/ARXIV.2601.07568}.
\newblock URL \url{https://doi.org/10.48550/arXiv.2601.07568}.

\bibitem[Ren et~al.(2026)Ren, Huang, Yuan, Zhao, and Liu]{ren2026topd}
Haolin Ren, Ziyang Huang, Chenhao Yuan, Jun Zhao, and Kang Liu.
\newblock Trace-based on-policy distillation for masked diffusion language models.
\newblock \emph{CoRR}, abs/2607.16872, 2026.
\newblock \doi{10.48550/ARXIV.2607.16872}.
\newblock URL \url{https://doi.org/10.48550/arXiv.2607.16872}.

\bibitem[Su et~al.(2026)Su, Helwig, Parashar, Chagi, Jotsna, Zhi, Caverlee, Kalathil, and Ji]{su2026opdlm}
Xingyu Su, Jacob Helwig, Shubham Parashar, Atharv Chagi, Lakshmi Jotsna, Degui Zhi, James Caverlee, Dileep Kalathil, and Shuiwang Ji.
\newblock Data-efficient autoregressive-to-diffusion language models via on-policy distillation.
\newblock \emph{CoRR}, abs/2606.06712, 2026.
\newblock \doi{10.48550/ARXIV.2606.06712}.
\newblock URL \url{https://doi.org/10.48550/arXiv.2606.06712}.

\bibitem[Wang et~al.(2026)Wang, Yuan, Zhong, Zhang, Xiao, Sun, and Qi]{wang2026bpm}
Hao Wang, Kun Yuan, Wenlin Zhong, Minglei Zhang, Han Xiao, Ming Sun, and Honggang Qi.
\newblock Cross-tokenizer on-policy distillation via byte-prefix marginalization.
\newblock \emph{CoRR}, abs/2607.22334, 2026.
\newblock \doi{10.48550/ARXIV.2607.22334}.
\newblock URL \url{https://doi.org/10.48550/arXiv.2607.22334}.

\bibitem[Wang et~al.(2025)Wang, Yang, Li, Tian, Shen, and Wang]{wang2026tracerl}
Yinjie Wang, Ling Yang, Bowen Li, Ye~Tian, Ke~Shen, and Mengdi Wang.
\newblock Revolutionizing reinforcement learning framework for diffusion large language models.
\newblock \emph{CoRR}, abs/2509.06949, 2025.
\newblock \doi{10.48550/ARXIV.2509.06949}.
\newblock URL \url{https://doi.org/10.48550/arXiv.2509.06949}.

\bibitem[Wu et~al.(2025)Wu, Li, Zhou, Lin, Gao, Yan, Yin, Bai, Xu, Chen, Chen, Tang, Zhang, Wang, Yang, Yu, Cheng, Liu, Li, Zhang, Meng, Wei, Ni, Chen, Cao, Peng, Qu, Wu, Wang, Yu, Wen, Feng, Xu, Wang, Zhang, Zhu, Wu, Cai, and Liu]{wu2025qwenimage}
Chenfei Wu, Jiahao Li, Jingren Zhou, Junyang Lin, Kaiyuan Gao, Kun Yan, Shengming Yin, Shuai Bai, Xiao Xu, Yilei Chen, Yuxiang Chen, Zecheng Tang, Zekai Zhang, Zhengyi Wang, An~Yang, Bowen Yu, Chen Cheng, Dayiheng Liu, Deqing Li, Hang Zhang, Hao Meng, Hu~Wei, Jingyuan Ni, Kai Chen, Kuan Cao, Liang Peng, Lin Qu, Minggang Wu, Peng Wang, Shuting Yu, Tingkun Wen, Wensen Feng, Xiaoxiao Xu, Yi~Wang, Yichang Zhang, Yongqiang Zhu, Yujia Wu, Yuxuan Cai, and Zenan Liu.
\newblock Qwen-image technical report.
\newblock \emph{CoRR}, abs/2508.02324, 2025.
\newblock \doi{10.48550/ARXIV.2508.02324}.
\newblock URL \url{https://doi.org/10.48550/arXiv.2508.02324}.

\bibitem[Wu et~al.(2026)Wu, Lan, Fu, Gao, Wang, Yu, Alvarez, Molchanov, Luo, Han, Zhu, and Xie]{wu2026fastdvlm}
Chengyue Wu, Shiyi Lan, Yonggan Fu, Sensen Gao, Jin Wang, Jincheng Yu, Jose~M. Alvarez, Pavlo Molchanov, Ping Luo, Song Han, Ligeng Zhu, and Enze Xie.
\newblock Fast-dvlm: Efficient block-diffusion vlm via direct conversion from autoregressive vlm, 2026.
\newblock URL \url{https://arxiv.org/abs/2604.06832}.

\bibitem[Yang et~al.(2025)Yang, Tian, Li, Zhang, Shen, Tong, and Wang]{yang2025mmada}
Ling Yang, Ye~Tian, Bowen Li, Xinchen Zhang, Ke~Shen, Yunhai Tong, and Mengdi Wang.
\newblock Mmada: Multimodal large diffusion language models.
\newblock In Danielle Belgrave, Cheng Zhang, Laura~N. Montoya, Hsuan{-}Tien Lin, Razvan Pascanu, Piotr Koniusz, Marzyeh Ghassemi, Nancy Chen, Iv{\'{a}}n Vladimir~Meza Ru{\'{\i}}z, and Arturo Loaiza{-}Bonilla (eds.), \emph{Advances in Neural Information Processing Systems 38: Annual Conference on Neural Information Processing Systems 2025, NeurIPS 2025, San Diego, CA, USA, December 2-7, 2025 / Mexico City, Mexico, November 30 - December 5, 2025}, 2025.
\newblock URL \url{http://papers.nips.cc/paper\_files/paper/2025/hash/caa934a507a952698d54efb24845fc4b-Abstract-Conference.html}.

\bibitem[You et~al.(2025)You, Nie, Zhang, Hu, Zhou, Lu, Wen, and Li]{you2025lladav}
Zebin You, Shen Nie, Xiaolu Zhang, Jun Hu, Jun Zhou, Zhiwu Lu, Ji{-}Rong Wen, and Chongxuan Li.
\newblock Llada-v: Large language diffusion models with visual instruction tuning.
\newblock \emph{CoRR}, abs/2505.16933, 2025.
\newblock \doi{10.48550/ARXIV.2505.16933}.
\newblock URL \url{https://doi.org/10.48550/arXiv.2505.16933}.

\bibitem[Yu et~al.(2024)Yu, Lezama, Gundavarapu, Versari, Sohn, Minnen, Cheng, Gupta, Gu, Hauptmann, Gong, Yang, Essa, Ross, and Jiang]{yu2024magvitv2}
Lijun Yu, Jos{\'{e}} Lezama, Nitesh~Bharadwaj Gundavarapu, Luca Versari, Kihyuk Sohn, David Minnen, Yong Cheng, Agrim Gupta, Xiuye Gu, Alexander~G. Hauptmann, Boqing Gong, Ming{-}Hsuan Yang, Irfan Essa, David~A. Ross, and Lu~Jiang.
\newblock Language model beats diffusion - tokenizer is key to visual generation.
\newblock In \emph{The Twelfth International Conference on Learning Representations, {ICLR} 2024, Vienna, Austria, May 7-11, 2024}. OpenReview.net, 2024.
\newblock URL \url{https://openreview.net/forum?id=gzqrANCF4g}.

\bibitem[Zeng et~al.(2026)Zeng, Yao, Liao, Tao, Liu, and Wang]{zeng2025diffusionvl}
Lunbin Zeng, Jingfeng Yao, Bencheng Liao, Hongyuan Tao, Wenyu Liu, and Xinggang Wang.
\newblock Diffusionvl: Translating any autoregressive models into diffusion vision language models.
\newblock In Paolo Favaro, Zuzana Kukelova, Atsuto Maki, Anna Rohrbach, Konrad Schindler, and Federico Tombari (eds.), \emph{Computer Vision -- ECCV 2026}, pp.\  331--348, Cham, 2026. Springer Nature Switzerland.
\newblock ISBN 978-3-032-37369-4.

\bibitem[Zhang et~al.(2025)Zhang, Zhang, Liang, Meng, Chen, Xu, and Zhou]{zhang2025dskd}
Xue Zhang, Songming Zhang, Yunlong Liang, Fandong Meng, Yufeng Chen, Jinan Xu, and Jie Zhou.
\newblock A dual-space framework for general knowledge distillation of large language models, 2025.
\newblock URL \url{https://arxiv.org/abs/2504.11426}.

\bibitem[Zhu et~al.(2025)Zhu, You, Xing, Huang, Liu, Zhuang, Lu, Wang, Wang, Wei, Guo, Hu, Ye, Chen, Li, Tang, Feng, Hu, Zhou, Zhang, Lan, Zhao, Zheng, Li, Li, and Wen]{zhu2025lladamoe}
Fengqi Zhu, Zebin You, Yipeng Xing, Zenan Huang, Lin Liu, Yihong Zhuang, Guoshan Lu, Kangyu Wang, Xudong Wang, Lanning Wei, Hongrui Guo, Jiaqi Hu, Wentao Ye, Tieyuan Chen, Chenchen Li, Chengfu Tang, Haibo Feng, Jun Hu, Jun Zhou, Xiaolu Zhang, Zhenzhong Lan, Junbo Zhao, Da~Zheng, Chongxuan Li, Jianguo Li, and Ji{-}Rong Wen.
\newblock Llada-moe: {A} sparse moe diffusion language model.
\newblock \emph{CoRR}, abs/2509.24389, 2025.
\newblock \doi{10.48550/ARXIV.2509.24389}.
\newblock URL \url{https://doi.org/10.48550/arXiv.2509.24389}.

\bibitem[Zhu et~al.(2026)Zhu, Guo, Choi, Molodyk, Yuan, Tao, and Chen]{zhu2025dmpo}
Yuchen Zhu, Wei Guo, Jaemoo Choi, Petr Molodyk, Bo~Yuan, Molei Tao, and Yongxin Chen.
\newblock Enhancing reasoning for diffusion llms via distribution matching policy optimization, 2026.
\newblock URL \url{https://arxiv.org/abs/2510.08233}.

\end{thebibliography}
\bibliographystyle{iclr2027_conference}
\clearpage
\appendix
\begin{center}
    {\Large\bf APPENDIX}
\end{center}
\section{Experimental Settings and Reproducibility}
\label{app:experiments}

\subsection{Data, endpoints, and benchmark separation}

\subsubsection{Endpoint construction and benchmark separation}

\paragraph{Reference set and matching rules.}
To keep endpoint training separate from evaluation, we compare all 47,628
Vision-SR1 source rows with 22,849 reference rows from
MMMU validation, MMMU-Pro Standard-10, MMStar, MMBench DEV-EN v1.1,
POPE, MME, GQA, SEEDBench, MathVerse mini-vision, MathVista testmini, GPQA main
and diamond, and the validation and test splits of M3CoT and ScienceQA.  Exact
matching uses source and content identifiers, normalized question/options and
prompt/options, canonical decoded RGB image hashes, and image-question pairs.  Text
normalization removes output-format suffixes, applies Unicode NFKC and
lowercasing, and retains Unicode word tokens with normalized option labels.
Image hashes include image dimensions and decoded RGB pixels.

Near matching uses equal normalized core text, word 3-gram Jaccard similarity
at least 0.90 with two boundary tokens on each side, or a 64-bit DCT perceptual
hash with Hamming distance at most four.  Exact and near matches are excluded
before endpoint construction.  Exact rules identify 1,203 rows, and near rules
remove another 1,878, for 3,081 exclusions and 44,547 remaining training
examples.  The final endpoint manifest contains these 44,547 examples.  The
overlap-check bundle records reference split counts, row-level exclusion
reasons, source positions, and SHA-256 hashes in
\path{AUDIT_REPORT.json} and \path{EXCLUDED_ROWS.jsonl}.

\paragraph{Endpoint identity and adaptation.}
Every endpoint contains one final-answer marker, preserves the verified answer,
fits the corresponding student canvas, and excludes chat special tokens and
empty rationales.  The cross-backbone manifests share source examples, prompts,
images, teacher responses, verified answers, membership, and ordering.
Student-specific tokenization and deterministic canvas fitting determine the
token sequences while retaining each final-answer suffix.

\subsubsection{Cross-tokenizer endpoint adaptation}
\label{app:endpoint_adaptation}

We remove model-specific wrappers from each teacher response and serialize the
result as \texttt{Reason:} followed by \texttt{Final answer:}.  The extracted
final answer must match the released training answer.  When it does, we retain
the rationale under the canonical serialization; otherwise, one
answer-conditioned regeneration supplies a deterministic task-aware rationale
while preserving the verified answer.  The resulting endpoint text is frozen
and shared across student architectures.

Each student then tokenizes that text with its own tokenizer.  Let
$L_\student$ be the response-canvas capacity including one terminal EOS,
$B_\student=L_\student-1$, and $f$ the student-token suffix beginning at the
final \texttt{Final answer:} marker after trailing EOS tokens are removed.  For
DiMoE-VL, the capacity adapter verifies that $f$ is an exact suffix of $z$ and
$|f|\leq B_\student$, then
applies
\begin{equation}
C_{B_\student}(z;f)=
\begin{cases}
z, & |z|\le B_\student,\\
z_{1:B_\student-|f|}\Vert f, & |z|>B_\student,
\end{cases}
\qquad
Y^\teacher=C_{B_\student}(z;f)\Vert[\mathrm{EOS}]_\student.
\label{eq:endpoint_capacity}
\end{equation}
The adapter leaves a fitting DiMoE-VL retokenization unchanged.  If the sequence
is too long, it preserves the verified final-answer suffix and truncates only
the preceding rationale.  LLaDA-V and LaViDa-LLaDA use 128-token canonical
response canvases, shortening only the reason field when needed and preserving
the terminal answer.  MMaDA retokenizes the shared text onto its 512-token
understanding canvas; its image adapter uses frozen MAGVITv2
(\Cref{app:backbone_interfaces,app:mmada_joint_training}).  Each text endpoint
appends exactly one
architecture-native terminal token and excludes mask tokens and other forbidden
special tokens.

\subsection{Models, training, and evaluation protocols}

\subsubsection{Cross-backbone CT-OPD configuration}

To complete the final global batch, we repeat the first 13 endpoints in the
fixed training order, each contributing all four state exposures.

\begin{table}[H]
\caption{Shared \method{} configuration across diffusion-VLM backbones.}
\centering
\resizebox{\textwidth}{!}{%
\begin{tabular}{lccc}
\toprule
Setting & LLaDA-V & LaViDa-LLaDA & \model{} \\
\midrule
Training data & Vision-SR1-47K & Vision-SR1-47K & Vision-SR1-47K \\
Trajectory-mask ratios & $1,.75,.5,.25$ & $1,.75,.5,.25$ & $1,.75,.5,.25$ \\
Endpoint-only mask slots & $1,1,1,1$ & $1,1,1,1$ & $1,1,1,1$ \\
Global batch & 16 & 16 & 16 \\
Learning rate & $3\times10^{-7}$ & $3\times10^{-7}$ & $3\times10^{-7}$ \\
Response canvas & 128 & 128 & 128 \\
Native reverse domain & full-canvas masked & full-canvas masked & full-canvas masked \\
Reverse steps per active domain & 32 & 32 & 32 \\
Native block/commit width & 4 tokens & 4 tokens & 4 tokens \\
Supervision & all unresolved & all unresolved & all unresolved \\
Seed & 3407 & 3407 & 3407 \\
Precision & BF16 & BF16 & BF16 \\
\bottomrule
\end{tabular}
}
\end{table}

\paragraph{Random-mask supervision.}
Random uses the same teacher endpoints, initialization, optimizer, and update
schedule as CT-OPD. For each example and training cycle, the native rollout
provides the number of supervised positions at each stage. A seeded random
permutation of endpoint-valid positions defines nested masks of those sizes;
visible positions and targets are then populated from the same endpoint.
LLaDA-V, LaViDa-LLaDA, and \modelbase{} randomize within the response canvas.
MMaDA randomizes within the active decision block for understanding and across
the image-token canvas for generation, preserving the native state outside
that support. Each mask tuple supplies four stage-specific updates under the
backbone's native schedule. Thus, Random retains partial-state training and coherent endpoint
supervision while replacing the student's choice of unresolved positions.

\paragraph{OPDLM+BPM distribution distillation.}
\label{app:opdlm_bpm}
We combine OPDLM \citep{su2026opdlm} with Byte-Prefix Marginalization
\citep{wang2026bpm} for visual understanding with different teacher and student
tokenizers. The diffusion student generates its own reverse trajectory on an
image-question prompt. The frozen Qwen3-VL-32B-Instruct teacher reads the same
image and question and scores causal prefixes of the student's completed
response. BPM maps these predictions to student response positions and
vocabulary entries through their byte representations, accounting for unequal
token boundaries and retaining unmatched probability in a residual category.
The student minimizes forward KL at unresolved response positions in sampled
intermediate states, retaining the visible tokens from its own rollout.
This gives OPDLM's causal-prefix supervision a cross-tokenizer interface,
whereas CT-OPD reconstructs both visible context and targets from a completed
teacher response. \Cref{tab:ctopd_unified_main} reports OPDLM+BPM on LLaDA-V,
LaViDa-LLaDA, and \modelbase{} under the same nine-benchmark evaluation.

\paragraph{Trajectory-mask source and pairing.}
The LLaDA-V intervention crosses trajectory-mask source
(current-student traces refreshed once per four-update cycle, or one
initial-Base four-stage mask tuple collected per prompt before training and reused)
with prompt pairing (the originating prompt, or a deterministic cross-prompt
derangement).  All four cells use the same 44,547 teacher endpoints, seeded
endpoint order, $\{100,75,50,25\}\%$ stages, 178,240 state exposures, 11,140
optimizer steps, global batch 16, and learning rate $3\times10^{-7}$.
Derangement is performed within endpoint active length and transports the
donor reveal order to the recipient's stage cardinality; it preserves stage
labels, nested mask structure, and every recipient mask count.  Only the
four-stage boolean mask tuple is transferred, never donor token values.  Thus
the two factors directly test whether mask placement must adapt online and
remain prompt specific under identical endpoint supervision and optimization.

\paragraph{Direct Trace-Target reconstruction control.}
On LLaDA-V, Direct Trace-Target uses the same 44,547 teacher endpoints,
seeded endpoint order, 32-step current-student rollout,
$\{100,75,50,25\}\%$ trajectory-mask ratios, 178,240 scheduled state
exposures, 11,140 optimizer steps, global batch 16, and learning rate
$3\times10^{-7}$ as CT-OPD.
It leaves each collected raw state unchanged and applies the same endpoint-token
cross-entropy at target-valid raw masked positions.  CT-OPD instead replaces
raw visible values with their teacher-endpoint values and clears coordinates
outside the endpoint span.  The comparison holds the endpoint and training
schedule fixed and uses the same online mask-selection protocol; each arm's
traces evolve with its own current checkpoint.

\subsubsection{Sparse DiMoE-VL architecture and training}

\begin{table}[H]
    \caption{\model{} parameter accounting.  Active parameters include
    always-active modules, the runtime visual tower, and the top-8 experts in
    each routed layer.}
    \label{tab:param_accounting}
    \centering
    \begin{tabular}{lrr}
        \toprule
        Component & Total & Active \\
        \midrule
        Language model & 7.357B & 1.720B \\
        Projector and image newline & 0.007B & 0.007B \\
        Routing auxiliaries & 0.000173B & 0.000173B \\
        SigLIP2 visual tower & 1.107B & 1.107B \\
        \midrule
        \textbf{DiMoE-VL} & \textbf{8.471B} & \textbf{2.834B} \\
        \bottomrule
    \end{tabular}
\end{table}

\begin{table}[t]
\caption{\textbf{Three-stage training recipe for \model{}.}  Stage~II produces
\modelbase{}; Stage~III applies \method{} to the resulting native
masked-diffusion policy.  A dash marks a component unused in that stage.}
\label{tab:dimoe_three_stage_training}
\centering
\footnotesize
\setlength{\tabcolsep}{3.2pt}
\renewcommand{\arraystretch}{1.07}
\begin{tabularx}{\textwidth}{@{}l*{3}{>{\centering\arraybackslash}X}@{}}
\toprule
\textbf{Configuration} & \textbf{Stage I} & \textbf{Stage II} & \textbf{Stage III} \\
\midrule
Training objective
  & Vision-language alignment
  & Full multimodal SFT
  & Counterfactual trace on-policy distillation \\
Optimization target
  & Multimodal projector
  & Projector and LLaDA-MoE decoder
  & Endpoint-token likelihood on masks from student traces \\
Trainable modules
  & Projector only
  & Projector and LLaDA-MoE decoder
  & LLaDA-MoE decoder, sparse routers/experts, and LM head \\
\midrule
Training data
  & LLaVA-Pretrain
  & MAmmoTH-VL SI-10M + OV-2M
  & Vision-SR1-47K endpoints \\
Endpoint teacher
  & -
  & -
  & Qwen3-VL-32B-Instruct \\
Vision encoder
  & \multicolumn{3}{c}{SigLIP2-Giant-Patch16-384 (frozen)} \\
Language backbone
  & \multicolumn{3}{c}{LLaDA-MoE-7B-A1B-Instruct} \\
Projector
  & \multicolumn{3}{c}{MLP-2$\times$-GELU multimodal projector} \\
Training extent
  & One alignment pass
  & One multimodal SFT pass
  & One endpoint pass over four trajectory-mask ratios \\
Learning rate
  & $1\times10^{-3}$
  & $1\times10^{-5}$
  & $3\times10^{-7}$ \\
LR schedule
  & cosine
  & cosine
  & cosine \\
Warmup ratio
  & $3\%$
  & $3\%$
  & $3\%$ \\
Per-device batch size
  & 16
  & 1
  & 1 \\
Sequence layout
  & 8,192 tokens
  & 8,192-16,384 tokens
  & native prompt + 128-token response canvas \\
Reverse-state construction
  & -
  & native SFT corruption
  & 32-step student trace; 4-token commits \\
Trajectory-mask ratios
  & -
  & -
  & $\{100,75,50,25\}\%$ \\
Distributed optimization
  & ZeRO-2
  & ZeRO-3
  & ZeRO-3 \\
Precision
  & \multicolumn{3}{c}{BF16} \\
\bottomrule
\end{tabularx}
\end{table}

Stage~I trains only the multimodal projector.  Stage~II jointly tunes the
projector and LLaDA-MoE decoder, producing \modelbase{} at update 80,000.
Stage~III freezes the visual encoder, projector, and input embeddings.  One
current-student trajectory supplies four masks at the prescribed ratios, each
applied to the corresponding endpoint.  Mean endpoint-token negative
log-likelihood over all unresolved positions updates the decoder, sparse
routers and experts, and LM head.  Stage~III optimizes this endpoint-token
objective directly.

\subsubsection{MMaDA training protocol}
\label{app:mmada_joint_training}

\paragraph{Prompt source and teacher images.}
Image-training prompts come from MaskGRPO's released GenEval training metadata
\citep{ma2026maskgrpo}, stored in
\path{dataset/GenEval/train_metadata_filtered.jsonl}.  The source contains
17,690 entries and 12,158 distinct prompt texts after stripping surrounding
whitespace.  A deterministic SHA-256 ordering with seed 3408 selects 8,000
entries, covering 5,629 distinct prompt texts.  With seed 3407, the understanding
branch selects 8,000 answer-anchored examples from the overlap-filtered
Vision-SR1 pool.
Each selected image prompt receives a Qwen-Image-2512 endpoint at
$512\times512$ resolution, using 50 inference steps, true CFG 4.0, and
a sample-specific seed derived from base seed 3407 and its sample ID.
Frozen MAGVITv2 maps each endpoint to 1,024 image tokens from an
8,192-entry codebook.

\paragraph{Joint update schedule.}
The CT-OPD run performs 8,000 updates with one shared AdamW optimizer.  It
alternates understanding and generation in a fixed $1{:}1$ ratio, beginning
with understanding, so each branch receives 4,000 updates.  Eight ranks process
one example each per update.  Every selected endpoint appears at four
trajectory-mask ratios, yielding 32,000 example-state exposures per branch.
We use learning rate $3\times10^{-7}$, zero weight decay, $3\%$ warmup,
cosine decay, and gradient clipping at 1.0.  Within each branch, Endpoint-only
uses the same endpoint IDs and four-slot exposure schedule, with a fully masked
($100\%$) canvas in every slot. Random follows the same joint schedule and
matches the supervised-position count at each stage in both branches.
The training manifests record endpoint
membership and schedule hashes.

\paragraph{Generation evaluation.}
The evaluation metadata contain 553 prompts, with four samples per prompt
(2,212 images), 50 reverse steps, and guidance scale 3.5.  Stripped prompt
texts in the selected training and evaluation sets are disjoint under exact
text matching.  GenEval overall is the arithmetic mean of its six category
scores; CLIP and ImageReward average over the same 2,212 generated images.

\subsubsection{RL-specific configurations}
\label{app:rl_configuration}

Within each backbone, MDPO and MaskGRPO start from the same pretrained
checkpoint, train the same modules as CT-OPD, and retain the native multimodal
prompt, prediction vocabulary, and response canvas.  Run length is set per
method so aggregate student compute matches CT-OPD within each backbone;
\Cref{tab:rl_configuration} lists its algorithm-specific settings.

\begin{table}[H]
\centering
\caption{\textbf{RL-specific training settings.}  MDPO and MaskGRPO use
128-token understanding canvases.  UniGRPO jointly trains MMaDA understanding
and image generation; paired entries denote text and image settings.}
\label{tab:rl_configuration}
\footnotesize
\setlength{\tabcolsep}{4pt}
\renewcommand{\arraystretch}{1.08}
\begin{tabularx}{\textwidth}{@{}l*{3}{>{\centering\arraybackslash}X}@{}}
\toprule
Setting & MDPO & MaskGRPO & UniGRPO \\
\midrule
Optimizer & AdamW & AdamW & AdamW \\
Learning rate & $3\times10^{-7}$ & $3\times10^{-7}$ & $3\times10^{-7}$ \\
Weight decay & 0 & 0 & 0 \\
LR schedule & constant & constant & cosine \\
Gradient clipping & 1.0 & 1.0 & 1.0 \\
Candidates per prompt & 4 & 4 & 4 \\
Reverse steps & 32 & 32 & 256 / 16 \\
Response tokens & 128 & 128 & 512 / 1,024 \\
Sampling temperature & 0.1 & 0.9 & 0.9 / 1.0 \\
Surrogate clipping $\epsilon$ & 0.2 & 0.2 & 0.2 \\
KL coefficient $\beta$ & 0 & 0 & 0.01 \\
Training-state selection & 4 advantage-selected states & 12 mask levels & 4 mask levels \\
Minimum replay mask rate & - & 0.5 & 0.5 \\
Seed & 3407 & 3407 & 3407 \\
Precision & BF16 & BF16 & BF16 \\
\bottomrule
\end{tabularx}
\end{table}

\paragraph{MDPO.}
Four candidate trajectories provide intermediate full-sequence predictions.
Each prediction is scored against the verified final answer; reward differences
and future reward averages then define the intermediate advantages.  After
per-prompt, per-step group normalization, MDPO selects four states by absolute
advantage mass and applies its confidence-weighted clipped objective.  Its
rollout reveals high-confidence tokens first, leaving lower-confidence
positions unresolved.

\paragraph{MaskGRPO.}
Four completed responses receive terminal answer-correctness rewards followed
by per-prompt group normalization.  The released \texttt{uni} schedule cycles
through 12 replay mask rates from $0.5+0.5/12$ to $0.999$, with random
remasking constructing the state for the tokenwise clipped objective.

\paragraph{UniGRPO.}
Understanding and generation alternate through one shared model with
$3\%$ warmup and cosine learning-rate decay.  Text responses receive
answer-correctness rewards, while image candidates receive an
equal-weight combination of group-standardized CLIP and ImageReward scores.
Each four-candidate group is reused over four branch updates with replay mask
rates $0.625$, $0.75$, $0.875$, and $0.999$.  The clipped surrogate includes a
$\beta=0.01$ K3 penalty to the policy frozen at group collection.  Image
rollouts use classifier-free guidance 3.0.

\paragraph{Answer scoring and implementation.}
Understanding prompts request a boxed final answer.  The reward parser also
accepts a concise option letter or numerical answer and compares the extracted
answer with the verified target.  Our implementation follows the MDPO and
MaskGRPO releases pinned at commits \texttt{e6ae439} and \texttt{bc09711} for
candidate-state reward alignment, direct-answer extraction, group
normalization, and remasking; tensor-level comparisons reproduce their losses
and gradients.  Each dense and sparse campaign predesignates one final
checkpoint, which supplies all scores within its model row.

\section{Supplemental Results and Visualizations}
\label{app:supplemental_results}

\subsection{Unified-model generation results and matched visualizations}

\begin{table}[H]
\centering
\caption{\textbf{CT-OPD improves every reported image-generation metric over
the frozen MMaDA-8B base.}  GenEval reports the overall score and its six
compositional categories.  Parenthesized values are absolute changes over
MMaDA-8B-MixCoT.  Bold marks the best result for each metric.}
\label{tab:mmada_joint_generation}
\setlength{\tabcolsep}{3.0pt}
\renewcommand{\arraystretch}{1.12}
\arrayrulecolor{tableline}
\scriptsize
\resizebox{\textwidth}{!}{%
\begin{tabular}{@{}lrrrrrrrrr@{}}
\toprule
\multirow{2}{*}{\textbf{Method}}
& \multicolumn{7}{c}{\textbf{GenEval}}
& \multirow{2}{*}{\textbf{CLIP score}}
& \multirow{2}{*}{\textbf{ImageReward}} \\
\cmidrule(lr){2-8}
& Overall & Single & Two & Count & Colors & Position & Color attr. & & \\
\midrule
\rowcolor{tablebase}
MMaDA-8B-MixCoT
& 55.45 & 91.56 & 67.68
& 42.19 & 80.05 & 19.25
& 32.00 & 28.43 & 0.585 \\
\rowcolor{tablebase}
\quad + Endpoint-only Distillation
& \ctgain{60.70}{+5.25}
& \ctgain{96.25}{+4.69}
& \ctgain{74.49}{+6.81}
& \ctgain{42.81}{+0.62}
& \ctbestgain{81.65}{+1.60}
& \ctgain{23.50}{+4.25}
& \ctgain{45.50}{+13.50}
& \ctgain{28.99}{+0.56}
& \ctgain{0.732}{+0.147} \\
\rowcolor{tablebase}
\quad + UniGRPO
& \ctgain{59.90}{+4.44}
& \ctgain{95.31}{+3.75}
& \ctgain{76.77}{+9.09}
& \ctbestgain{43.44}{+1.25}
& \ctgain{80.85}{+0.80}
& \ctbestgain{24.50}{+5.25}
& \ctgain{38.50}{+6.50}
& \ctgain{28.81}{+0.38}
& \ctgain{0.752}{+0.167} \\
\rowcolor{tablebase}
\quad + Random
& \ctgain{60.06}{+4.60}
& \ctgain{95.00}{+3.44}
& \ctgain{76.01}{+8.33}
& \ctloss{40.00}{$-$2.19}
& \ctgain{80.59}{+0.53}
& \ctgain{22.75}{+3.50}
& \ctgain{46.00}{+14.00}
& \ctgain{29.01}{+0.58}
& \ctgain{0.765}{+0.180} \\
\rowcolor{tableours}
\quad\textbf{+ \method{}}
& \ctbestgain{61.69}{+6.23}
& \ctbestgain{96.88}{+5.31}
& \ctbestgain{77.53}{+9.85}
& \ctgain{42.81}{+0.63}
& \ctbestgain{81.65}{+1.60}
& \ctgain{24.25}{+5.00}
& \ctbestgain{47.00}{+15.00}
& \ctbestgain{29.16}{+0.73}
& \ctbestgain{0.774}{+0.189} \\
\bottomrule
\end{tabular}%
}
\arrayrulecolor{black}
\end{table}

The matched examples in \cref{fig:mmada_generation_qualitative} illustrate the
composition categories summarized in \cref{tab:mmada_joint_generation}.  With
the prompt, sample index, and sampling seed fixed, they compare entity
separation, attribute binding, cardinality, and spatial relations across the
four methods.

\begin{figure}[H]
    \centering
    \includegraphics[width=\textwidth]{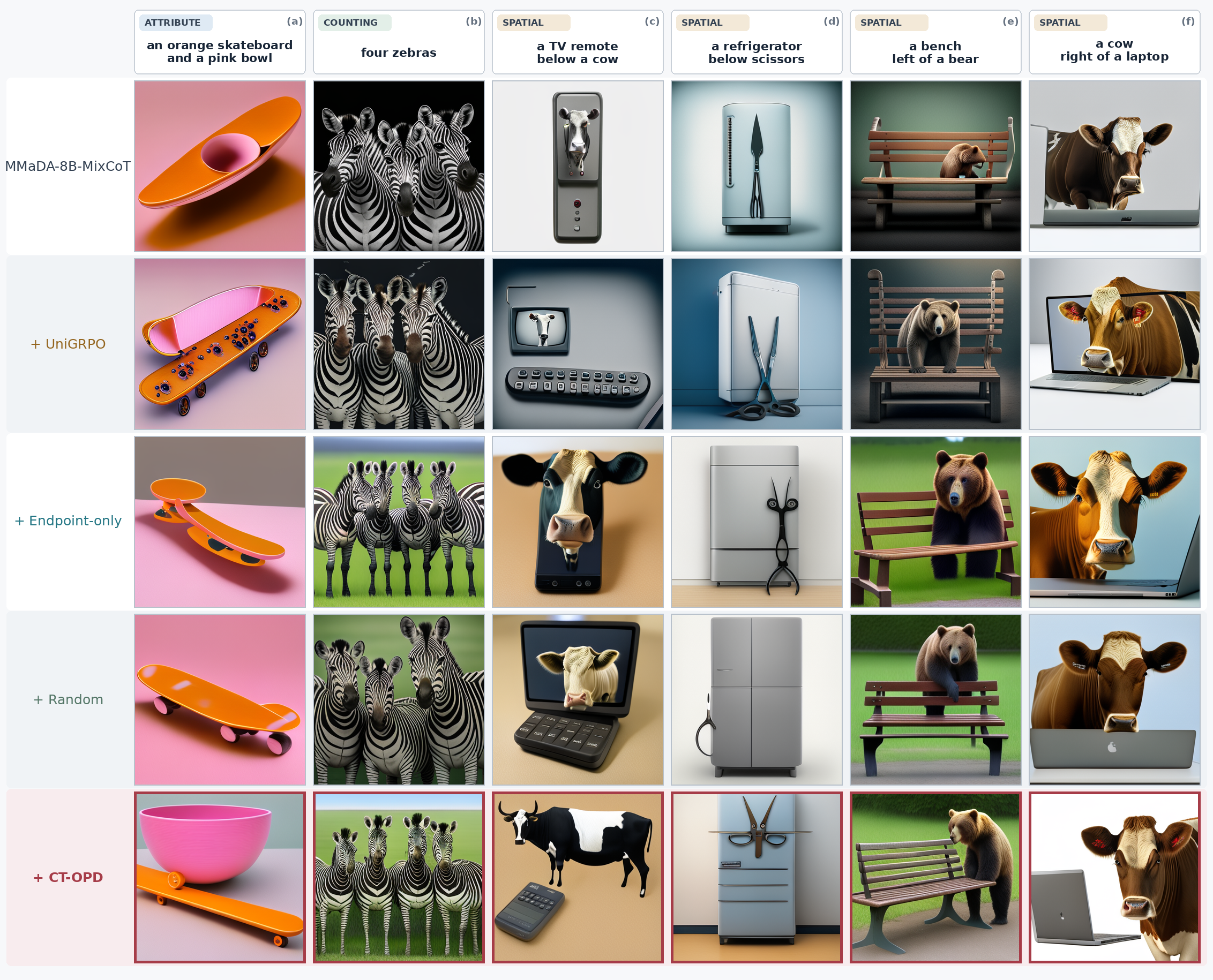}
    \caption{\textbf{Matched qualitative comparison on compositional image
    generation.}  Each column uses the same GenEval prompt, sample index, and
    sampling seed for MMaDA-8B-MixCoT, UniGRPO, Endpoint-only Distillation, and
    \method{}.  The examples display differences in attribute binding, object
    separation, cardinality, and relative placement.}
    \label{fig:mmada_generation_qualitative}
\end{figure}

\subsection{Additional prediction, routing, and comparison diagnostics}

\IfFileExists{generated/DIMOE_ROUTER_RESULTS_COMPLETE.tex}{
\subsubsection{Sparse routing after post-training}
\paragraph{Expert utilization and route overlap.}
\label{app:dimoe_router_integrity}
\begin{table}[H]
\centering
\caption{Paired DiMoE-VL routing behavior on identical base-model denoising states. Route statistics are aggregated over response tokens; paired changes compare \modelbase{} with the final CT-OPD checkpoint on the same states.}
\label{tab:dimoe_router_integrity}
\small
\begin{tabular}{lrr}
\toprule
Response-token route statistic & Base & +CT-OPD \\
\midrule
Normalized gate entropy & 0.964 & 0.965 \\
Soft effective experts (of 64) & 63.9 & 63.9 \\
Top-8 occupancy effective experts & 62.7 & 62.9 \\
Active expert fraction & 100.0\% & 100.0\% \\
Top-8 occupancy CV & 0.203 & 0.189 \\
\midrule
Paired base-to-final drift & \multicolumn{2}{c}{Value} \\
Gate Jensen-Shannon divergence & \multicolumn{2}{c}{0.0078} \\
Top-1 route agreement & \multicolumn{2}{c}{60.9\%} \\
Top-8 overlap & \multicolumn{2}{c}{72.7\%} \\
Exact top-8 set-change rate & \multicolumn{2}{c}{79.0\%} \\
\bottomrule
\end{tabular}
\end{table}

All 64 experts are active at both checkpoints.  On identical Base-model
denoising states, normalized gate entropy changes from $0.964$ to $0.965$, and
the coefficient of variation in top-8 occupancy changes from $0.203$ to
$0.189$.  Top-1 route agreement is $60.9\%$, while top-8 overlap is $72.7\%$.
The paired panel shows state-dependent routing while expert utilization remains
broad throughout the pool.
}{}

\IfFileExists{generated/CT_OPD_SIGNAL_COVERAGE_COMPLETE.tex}{
    \subsubsection{Signal coverage with equal terminal rewards}
    \label{app:signal_coverage}
\begin{table}[H]
\centering
\caption{Token supervision across denoising stages.  Coverage is measured on
common LLaDA-V evaluation states reconstructed from frozen Base trajectory
masks and teacher endpoints.  Zero-contrast prompts have identical correctness
rewards across four sampled completions.  Nonempty-state coverage is the
fraction of states with at least one supervised token; token coverage is the
fraction of endpoint tokens supervised at each stage.  Aggregate values pool
all four stages.}
\label{tab:ct_signal_coverage}
\small
\begin{tabular}{lrrrr}
\toprule
& \multicolumn{2}{c}{All prompts ($n=512$)} & \multicolumn{2}{c}{Zero-contrast prompts ($n=311$)} \\
\cmidrule(lr){2-3}\cmidrule(lr){4-5}
Mask ratio & Nonempty states (\%) & Token coverage (\%) & Nonempty states (\%) & Token coverage (\%) \\
\midrule
100\% & 100.00 & 100.00 & 100.00 & 100.00 \\
75\% & 100.00 & 73.08 & 100.00 & 73.83 \\
50\% & 86.91 & 45.02 & 90.68 & 47.22 \\
25\% & 70.51 & 21.13 & 75.24 & 22.11 \\
\midrule
Aggregate & 89.36 & 59.81 & 91.48 & 60.79 \\
\bottomrule
\end{tabular}
\end{table}

}{}

\subsubsection{Conditional prediction with zero terminal contrast}
\label{app:conditional_prediction_probe}

\paragraph{Zero-terminal-contrast panel.}
When group-relative terminal reward has no centered task contrast, this panel
measures conditional token learning across the final LLaDA-V checkpoints on
512 frozen Vision-SR1 training prompts.  Four Base screening completions per
prompt, sampled with seed 3407, define 292 all-wrong, 201 mixed-correctness, and
19 all-correct groups.  These fixed memberships define a common comparison of Base,
Endpoint-only, Random, Direct Trace-Target, and CT-OPD on common
conditional predictions.  Token evaluation uses the 292 all-wrong prompts.
Endpoint-only is the exposure-matched checkpoint trained
with four $100\%$-mask updates per cycle.

\paragraph{Paired common-state evaluation.}
A frozen Base trace, generated by deterministic low-confidence decoding at
temperature zero, supplies trajectory masks at ratios
$\{1.0,0.75,0.5,0.25\}$.  Teacher reconstruction gives every evaluated model
the same 2,048 state identities and target positions, isolating checkpoint
prediction on common inputs.  \Cref{tab:ct_signal_coverage} measures supervision
coverage on these states for all 512 prompts and the 311 zero-contrast prompts
(292 all-wrong and 19 all-correct).  The latter group contains 1,138 nonempty
states out of 1,244 and 63,087 supervised endpoint-token positions out of
103,780 across the four stages.  Token metrics aggregate the partial ratios
$\{0.75,0.5,0.25\}$: 876 states from the all-wrong
prompts, of which 784 have nonempty supervision, comprising 35,935 supervised
tokens.  We weight NLL and token accuracy by supervised token count.  For each
paired percentile 95\% interval, we resample prompt identities 10,000 times
with bootstrap seed 3407 and retain each prompt's partial states jointly across
models.  Every replicate recomputes token-weighted NLL and accuracy.  We apply
the same paired estimator separately at each mask ratio in the lower block of
\Cref{tab:zero_contrast_learning}.

On the all-wrong common-state probe, CT-OPD has lower NLL and higher token
accuracy than each of the four comparators
(\Cref{tab:zero_contrast_learning}).  Its NLL is 3.3661, compared with 3.5736
for Endpoint-only and 3.5202 for Random; its token accuracy is 27.43\%,
compared with 24.15\% and 26.12\%, respectively.  The prompt-paired intervals
quantify these differences on identical teacher-reconstructed partial states.
At 75\%, 50\%, and 25\% masking, the NLL and accuracy intervals favor CT-OPD;
all six exclude zero.

\paragraph{Canonical reconstruction versus raw-state targeting.}
The Direct/CT pair shares teacher endpoints, the online mask-selection
protocol, optimizer, update count, and state exposure; canonical reconstruction
is the controlled difference.  Each arm refreshes traces from its own current
checkpoint.  Direct Trace-Target retains the student's visible values, whereas
CT-OPD replaces them with the corresponding endpoint values.  We score both
final checkpoints on the same frozen teacher-reconstructed panel.  Across
all 512 prompts and the three partial ratios, NLL is 4.1664 for Direct
Trace-Target and 3.3392 for CT-OPD, while token accuracy is 13.59\% and 27.65\%.
The paired CT-minus-Direct differences are $-0.8272$ nats/token (95\% CI
$[-0.8717,-0.7847]$) and $+14.06$ points (95\% CI $[13.13,15.07]$).

\begin{table}[H]
\centering
\caption{\textbf{Canonical reconstruction and partial-state prediction.}
Direct Trace-Target preserves raw student-visible values while matching
CT-OPD's endpoints, online mask-selection protocol, optimizer, and exposure.
Entries report CT-OPD minus Direct Trace-Target with prompt-paired 95\% CIs.}
\label{tab:direct_trace_target_probe}
\arrayrulecolor{tableline}
\small
\setlength{\tabcolsep}{8pt}
\renewcommand{\arraystretch}{1.12}
\begin{tabular}{lcc}
\toprule
Mask & $\Delta$NLL (nats/token) $\downarrow$ & $\Delta$accuracy (pp) $\uparrow$ \\
\midrule
100\% & $+0.0384\;[0.0299,0.0469]$ & $-0.24\;[-0.44,-0.03]$ \\
75\%  & $-0.5086\;[-0.5531,-0.4657]$ & $+7.78\;[6.90,8.72]$ \\
50\%  & $-0.9615\;[-1.0334,-0.8925]$ & $+16.41\;[14.81,18.14]$ \\
25\%  & $-1.6434\;[-1.7449,-1.5460]$ & $+30.79\;[28.31,33.35]$ \\
\midrule
Partial aggregate & $-0.8272\;[-0.8717,-0.7847]$ & $+14.06\;[13.13,15.07]$ \\
\bottomrule
\end{tabular}
\arrayrulecolor{black}
\end{table}

The paired difference first favors CT-OPD when visible response context appears
and grows as the state becomes more resolved.  On the separate nine-benchmark
evaluation, Direct Trace-Target has average 45.99 and CT-OPD has 49.79.
Together, the common-state probe and benchmark average show the benefit of
pairing the student's mask geometry with endpoint-consistent visible values
before applying endpoint targets.

\IfFileExists{generated/LOCAL_JUDGE_RESULTS.tex}{
    \subsubsection{Paired semantic evaluation}
    \input{generated/LOCAL_JUDGE_RESULTS.tex}
}{}

\IfFileExists{generated/MULTISCALE_RESULTS.tex}{
    \subsubsection{Four-stratum training dynamics}
    \input{generated/MULTISCALE_RESULTS.tex}
}{}

\IfFileExists{generated/CROSS_ARCHITECTURE_ACC_LOOSE_V2_COMPLETE.tex}{
    \IfFileExists{generated/CROSS_ARCHITECTURE_ACC_LOOSE_V2_PUBLIC_COMPLETE.tex}{
        \subsubsection{Public absolute positioning}
        The following comparison places final model predictions alongside a
        fixed inventory of public diffusion-VLM scores.
        \input{generated/CROSS_ARCHITECTURE_ACC_LOOSE_V2_PUBLIC_COMPARISON.tex}
        \input{generated/CROSS_ARCHITECTURE_ACC_LOOSE_V2_PUBLIC_PARAGRAPH.tex}
    }{}
}{}

\section{Method and Native-Objective Details}
\label{app:method}

\subsection{Canonical reconstruction and exact updates}

\subsubsection{Complete update cycle}
\label{app:algorithm}

Each cycle begins by adapting one AR VLM response into the student endpoint
$Y^\teacher$.  CT-OPD runs the current diffusion student once, records its raw
reverse trace, and applies four architecture-native ordinal selectors.  At
stratum $j$, it constructs
$\xi_j^{(a)}=P_j^{(a)}(Y^\teacher,G_j^{(a)}(Z^\student),c)$ and the corresponding
state $F^{(a)}(Y^\teacher,\xi_j^{(a)},c)$.  Four sequential updates apply the
native denoising log score to unresolved active targets, after which the next
cycle refreshes the rollout policy.  Trajectory-variable extraction is performed
under stop-gradient; each update differentiates only the native loss evaluated
on the reconstructed input.

\subsubsection{Exact empirical update operator}

We materialize the endpoint set once from Vision-SR1-47K as
$\mathcal E=\{(c_e,y_e,L_e)\}_{e=1}^{|\mathcal E|}$.  A rationale is retained
when its final answer is verified; otherwise, one answer-anchored regeneration
constructs a coherent endpoint around the verified answer.  Examples follow a
deterministic seeded order, and each trace supplies the four ordered
trajectory-derived strata used by the objective.

At cycle $k$, set $\bar\theta_k=\theta_{k,0}$ and generate one deterministic,
stop-gradient raw trace $Z_{e,k}$ per endpoint.  For $j=0,\ldots,3$, each rank
computes a per-example mean over that example's unresolved endpoint tokens;
the distributed scalar is the mean of the 16 per-example losses,
\begin{equation}
\widehat\cL_{k,j}^{(a)}(\theta;\bar\theta_k)
=\frac1{16}\sum_{e\in\mathcal B_k}
\mathbf 1\{|U_{e,k,j}^{(a)}|>0\}
\frac{1}{|U_{e,k,j}^{(a)}|\vee1}
\sum_{u\in U_{e,k,j}^{(a)}}
-\log p_{\theta,u}^{(a)}(y_{e,u}\mid X_{e,k,j}^{(a)}).
\label{eq:exact_empirical_operator}
\end{equation}
An example with an empty supervised set contributes a graph-connected zero and
retains its batch slot.  The four optimizer calls occur sequentially at
$\theta_{k,0},\ldots,\theta_{k,3}$ while sharing $Z_{e,k}$, so each stratum is
evaluated with the latest parameters.  AdamW uses learning rate
$3\times10^{-7}$, zero weight decay, 3\% warmup, cosine decay, and global
gradient norm 1.0.
For the exposure-matched Endpoint-only control, the same endpoint minibatch
receives the same four sequential optimizer calls with
$U_{e,k,j}^{(a)}=A(y_e)$ for every $j\in\{0,1,2,3\}$.  Its four mask slots are
therefore $\{100,100,100,100\}\%$.

\subsection{Backbone interfaces and native objectives}

\subsubsection{Backbone interfaces}
\label{app:backbone_interfaces}

CT-OPD maps each teacher endpoint into the student's prediction space and
reconstructs intermediate states under its native decoding structure.  Text
endpoints are retokenized for LLaDA-V, LaViDa-LLaDA, \model{}, and MMaDA,
while image endpoints are encoded into MMaDA's discrete image-token space.
\Cref{tab:backbone_interfaces} connects these adapters to their prediction
supports, training traces, and downstream results.  Each branch minimizes
mean cross-entropy over unresolved endpoint-valid positions.

\paragraph{Categorical prediction supports.}
The training probabilities are normalized over the following alphabets:
\begin{align}
\mathcal V_L
&=\{0,\ldots,126348\}\setminus
\{126080,126081,126336,126346,126347\},\\
\mathcal V_D
&=\{0,\ldots,157152\}\setminus
\{156891,156893,156895\},\\
\mathcal V_M
&=\{0,\ldots,134655\}.
\end{align}
LLaDA-V and LaViDa-LLaDA use $\mathcal V_L$, and \model{} uses
$\mathcal V_D$.  Both MMaDA branches apply cross-entropy over the complete
output vocabulary $\mathcal V_M$.  Image targets are the 8,192 MAGVITv2
codebook entries mapped to their model token IDs; generation sampling selects
among those image entries.  The prediction support of each training objective
stays fixed across denoising stages.

\begin{table}[H]
\centering
\caption{Backbone interfaces and native training traces.  Prediction supports
specify the normalization domain of the training loss; trace settings describe
the student rollouts used to select intermediate states.}
\label{tab:backbone_interfaces}
\small
\begin{tabularx}{\textwidth}{@{}l>{\raggedright\arraybackslash}X>{\raggedright\arraybackslash}Xll@{}}
\toprule
Backbone / branch & Endpoint adapter & Training trace & Support & Results \\
\midrule
LLaDA-V & Native text tokenizer & 128-token canvas; 32 steps & $\mathcal V_L$ & \Cref{tab:ctopd_unified_main} \\
LaViDa-LLaDA & Native text tokenizer & 128-token canvas; 32 steps & $\mathcal V_L$ & \Cref{tab:ctopd_unified_main} \\
\model{} & Native text tokenizer; answer-preserving fitting & 128-token canvas; 32 steps & $\mathcal V_D$ & \Cref{tab:ctopd_unified_main} \\
MMaDA / understanding & Native text tokenizer & 512-token canvas; 128-token blocks; 256 steps & $\mathcal V_M$ & \Cref{tab:ctopd_unified_main} \\
MMaDA / generation & Frozen MAGVITv2 & 1,024-token image canvas; 16 steps & $\mathcal V_M$ & \Cref{tab:mmada_joint_generation} \\
\bottomrule
\end{tabularx}
\end{table}

For MMaDA understanding, masks are selected within the first active response
block.  For generation, they span the image-token canvas.  The two branches
share the joint update schedule in \Cref{app:mmada_joint_training}; generation
evaluation uses 50 reverse steps and guidance scale 3.5.

\subsection{Structural diffusion specializations}

\subsubsection{Block-diffusion formulation}
\label{app:bd3_validity}

For block-diffusion decoders, CT-OPD reconstructs each active block from the
teacher endpoint while preserving the decoder's block-causal conditioning.
The endpoint-level objective aggregates supervision over unresolved positions
in all endpoint-overlapping blocks.  For the BD3 factorization, let $P$ be
the response offset, $B$ the block width,
$I_b$ the coordinates in block $b$, and
\begin{equation}
\mathcal B(Y,c)=
\left\{\left\lfloor\frac{P}{B}\right\rfloor,\ldots,
\left\lfloor\frac{P+|A(Y)|-1}{B}\right\rfloor\right\}.
\label{eq:bd3_endpoint_blocks}
\end{equation}
For stratum $j$, the block-local trajectory descriptor is
$\Xi_j^{\mathrm{BD3}}=(g_{b,j},r_{b,j}^{\mathrm{act}})_{b}$.  Each component
comes from a visited block-local trace under the frozen fixed-canvas,
no-early-EOS trajectory-recording policy.  One BD3 forward evaluates the
resulting factorized bundle.  Block $b$ contributes trajectory mask $g_{b,j}$ and
$U_{b,j}=A(Y)\cap I_b\cap\{i:g_{b,j,i}=1\}$.  Let
$\cC_b(Y,g_{b,j})$ denote the block-local analogue of
\Cref{eq:active_transport}: teacher tokens occupy resolved coordinates of
$I_b$, while unresolved coordinates are masked.  With
$N_j(Y)=\sum_{b\in\mathcal B(Y,c)}|U_{b,j}|$, the endpoint-wide score is
\begin{equation}
\cL_{j}^{\mathrm{BD3}}
=\E\!\left[
\mathbf 1\{N_j(Y)>0\}\frac{
\sum_{b\in\mathcal B(Y,c)}\sum_{i\in U_{b,j}}
    -\log p^{\mathrm{BD3}}_{\theta,i}\!\left(
Y_i^\teacher\mid c,Y_{I_{<b}}^\teacher,
\cC_b(Y^\teacher,g_{b,j})\right)}
{N_j(Y)\vee1}
\right].
\label{eq:bd3_multiblock_objective}
\end{equation}
Here $p^{\mathrm{BD3}}_{\theta,i}$ is the full-vocabulary softmax of the
block decoder.  The objective scores every endpoint-overlapping block in the released
offset-block-causal domain; an empty effective stratum contributes a
graph-connected zero.  CT-OPD retains the released conditional factorization
and uses the endpoint-token mean score in
\Cref{eq:bd3_multiblock_objective} instead of BD3 pretraining's $1/t_b$
weighting.

\subsubsection{General diffusion specializations}

\Cref{tab:specializations} instantiates the endpoint adapter
$A_{\teacher\to\student}$, corruption map $F_t$, native target map $\tau_t$,
and discrepancy $D_t$ for four diffusion families.  In each case, CT-OPD fixes
a current-student raw trace, extracts an endpoint-valid architecture-specific
trajectory variable, reconstructs the state-target pair at the teacher
endpoint, and evaluates the generator's native per-state loss.  For an
endpoint-universal corruption map, the selected trajectory variable itself
remains fixed.
Each row assumes a specified common-noise realization and an extractor that
reconstructs the source state.

\begin{table*}[h]
\caption{CT-OPD specializations under the unified corruption map.}
\label{tab:specializations}
\centering
\resizebox{\textwidth}{!}{%
\begin{tabular}{lllll}
\toprule
Family & Endpoint adapter & Trajectory variable $\xi_t$ & Transported state & Conditional score \\
\midrule
Absorbing masked & Detokenize-retokenize & Unresolved-position mask $g_t$ & $\cC(Y^\teacher,g_t)$ & Unresolved-token NLL \\
Categorical & Domain-specific endpoint map & Base noise and scheduler randomness & $F_t(Y^\teacher,\xi_t)$ & Denoising cross-entropy \\
Gaussian/latent & Frozen endpoint encoder & Trace residual $\epsilon_t^\student$ & $\alpha_tY^\teacher+\sigma_t\epsilon_t^\student$ & $x_0$, $\epsilon$, or $v$ regression \\
Block diffusion & Domain-specific endpoint map & Block/reveal/remasking variables & $F_t(Y^\teacher,\xi_t)$ & Active-block NLL \\
\bottomrule
\end{tabular}
}
\end{table*}

For Gaussian diffusion with $\sigma_t>0$, the trajectory residual may be stored
during sampling or reconstructed relative to the student's current clean
estimate:
\begin{equation}
    \epsilon_t^\student
    =\frac{S_t^\student-\alpha_t\hat Y_t^\student}{\sigma_t},
    \qquad
    \widetilde S_t
    =\alpha_tY^\teacher+\sigma_t\epsilon_t^\student.
    \label{eq:gaussian_transport}
\end{equation}
Together with $\hat Y_t^\student$, this endpoint-relative residual exactly
reconstructs the observed student state.  CT-OPD preserves the residual
coordinate while replacing the endpoint component of the transported state;
the reverse-process occupancy induces the residual distribution.  Under the
common variance-preserving parameterization, the corresponding native targets
are
\begin{equation}
    \tau_{x_0}=Y^\teacher,
    \qquad
    \tau_{\epsilon}=\epsilon_t^\student,
    \qquad
    \tau_v=\alpha_t\epsilon_t^\student-\sigma_tY^\teacher.
    \label{eq:gaussian_native_targets}
\end{equation}
When the transported target has finite second moment, CT-OPD supports native
$x_0$-, noise-, or velocity-prediction through a counterfactual state-target
coupling.  The categorical specialization requires a measurable abduction
kernel or right-inverse supported on
$F_t(\hat Y_t^\student,\cdot,c)^{-1}\{S_t^\student\}$.  All base transition
noise, position choices, scheduler variables, remasking decisions, and latent
transition variables needed by this representation jointly form
$\xi_t^\student$, yielding the exact replay
$F_t(Y^\teacher,\xi_t^\student)$.  In masked diffusion, this trajectory
variable reduces to the unresolved-coordinate indicator.

\subsubsection{Interfaces to on-policy diffusion distillation}
\label{app:distillation_interfaces}

On-policy distillation methods are defined by the teacher object they transfer.
TOPD matches the token distribution of a frozen diffusion teacher evaluated on
the student's masked state by reverse KL \citep{ren2026topd}.  dOPSD averages
stop-gradient predictions from later masked states of the same trajectory
\citep{dat2026dopsd}.  OPDLM builds causal supervision from student endpoints in
an aligned AR-diffusion model pair \citep{su2026opdlm}.  OPTD rolls a frozen
diffusion teacher forward from a student state and compresses outcome-preserving
future transitions for lower-NFE decoding \citep{lu2026optd}.
Our OPDLM+BPM baseline combines causal-prefix supervision with byte-space
distribution alignment across tokenizers (\Cref{app:opdlm_bpm}).

CT-OPD supplies a cross-tokenizer VLM interface: an AR VLM provides an
answer-verified text endpoint, while the diffusion VLM supplies a native
trajectory variable from its reverse process.  Canonical reconstruction
combines them in one student-vocabulary denoising problem.  The interface
therefore connects an AR VLM endpoint teacher to a diffusion VLM with a
different tokenizer, while the trajectory variable and denoising objective
remain native to the student.  The controls hold endpoint information,
native CE, and update exposure fixed while varying full masking, mask placement,
or state reconstruction.  Together, they separate the effects of partial states,
student-selected positions, and canonical reconstruction under a common
teacher-information budget.

\section{Theory and Proofs}
\label{app:theory}

\subsection{Objective validity and learning signal}

\subsubsection{Exact projected occupancy law and native support}
\label{app:product_coupling}

Fix the cycle-start policy $\bar\theta$.  The response canvas and each native
prediction alphabet are finite.  Contexts and scorer inputs take values in
standard Borel spaces, ensuring the existence of the conditional laws used
below.  Endpoint adaptation is a measurable function of the teacher response
and training example.  Trace collection uses the original context and random
bits independent of endpoint generation.  Consequently, conditional on
$(c,a^*)$ and a selected stage $j$, the adapted endpoint $Y$ and raw trace $Z$
are independent.  Materialized endpoints and deterministic samplers correspond
to point-mass laws.

The raw trajectory mask is defined on the full canvas before restricting it to
the endpoint's active span.  If native metadata $H_j$ are not determined by the
unresolved set $M_j$, we regard $(M_j,H_j)$ as an augmented mask record and use
its joint law; reconstruction acts on the mask component.  Any subsequent
endpoint-dependent selection belongs to the deterministic architecture
adapter.  Thus independent raw sources can produce a coupled active target
set $U=A(Y)\cap M_j$ and coupled scoring input.

For a bounded measurable test function $f$, the product law gives
\begin{align}
&\E[f(Y,M_j,\cC(Y,M_j))\mid c,a^*,j] \nonumber\\
&\quad=\int\!\int f(y,m,\cC(y,m))\,
 \pi_\teacher(dy\mid c,a^*)\rho_{\bar\theta,j}(dm\mid c).
\end{align}
This proves the exact pushforward law for endpoints and trajectory masks in
\Cref{eq:projected_occupancy_law}.  The same identity holds with the augmented
mask record and the architecture adapter inserted into $f$.  Deterministic
projection turns the independent source draws into CT-OPD's coupled
reconstructed inputs.

\begin{proposition}[Canonical reconstruction from a trajectory mask]
\label{prop:occupancy_projection}
Fix a legal endpoint $y$ and trajectory mask $m\subseteq\mathcal I$.  The state
$\cC(y,m)$ is the unique canvas with unresolved active set $m\cap A(y)$,
visible values equal to $y$, and masked inactive positions.  It lies in the
support of any absorbing corruption law that assigns this active mask pattern
positive probability at the chosen denoising stage.
\end{proposition}

\begin{proof}
Let $y$ be a legal endpoint, so its active tokens differ from $\masktok$, and
let $m\subseteq\mathcal I$.  The sets
$m\cap A(y)$, $A(y)\setminus m$, and $\mathcal I\setminus A(y)$ form a
partition of $\mathcal I$.  The required value on the first set is
$\masktok$, on the second it is $y_i$, and on the third it is $\masktok$ by
the inactive-canvas convention.  These requirements determine every entry
and agree with \Cref{eq:active_transport}, proving existence and uniqueness.

Put $U=m\cap A(y)$.  An absorbing corruption of $y$ produces exactly this
canvas whenever its masked active set is $U$.  Therefore any subset law
assigning positive probability to $U$ assigns positive probability to
$\cC(y,m)$.  In particular, independent active-coordinate corruption at
$0<\alpha_t<1$ gives probability
$(1-\alpha_t)^{|U|}\alpha_t^{|A(y)|-|U|}>0$.  Inactive positions are
deterministic and introduce no additional event.  For $A(y)=\varnothing$,
the empty-product probability is one.  At $\alpha_t=0$, the support contains
only the fully masked active span; at $\alpha_t=1$, it contains only the
fully visible active span.  The proposition's positive-probability condition
includes both boundaries.
\end{proof}

\paragraph{Native block support.}
For a block decoder, fix $y$, the prompt offset, and a response block $b$.
The native conditioning variables are the prompt, earlier clean response
blocks $y_{I_{<b}}$, and the partially resolved current block
$\cC_b(y,g_{b,j})$.  Prompt coordinates within a boundary block are held
fixed.  Applying the coordinate argument to $A(y)\cap I_b$ proves uniqueness
and support whenever the block corruption assigns positive probability to
$U_{b,j}$.  A product reference law of these block-local corruptions gives
the entire bundle positive mass when every component has positive mass.
Each factor in \Cref{eq:bd3_multiblock_objective} therefore has valid native
conditioning.  Earlier blocks supply clean endpoint context for a later block,
while retaining their own corrupted queries elsewhere in the bundle.  The
recorded block-local queries and endpoint prefixes instantiate this conditional
factorization.

\subsubsection{Reverse-event equivalence for native categorical heads}
\label{app:properness}

\paragraph{Absorbing posterior.}
Let $V_t$ be one coordinate of an absorbing forward process with clean token
$Y=y\ne\masktok$.  Survival at time $t$ has probability $\alpha_t$.
For $s<t$ with $0\leq\alpha_t<\alpha_s\leq1$, absorption implies
$\{V_t=y\}\subseteq\{V_s=y\}$.  Hence
\begin{align}
\Pr(V_s=y,V_t=\masktok\mid Y=y)&=\alpha_s-\alpha_t,\\
\Pr(V_s=y\mid V_t=\masktok,Y=y)
&=\frac{\alpha_s-\alpha_t}{1-\alpha_t}=\gamma.
\end{align}
The conditional mass $1-\gamma$ corresponds to persistence of the mask.
The absorbing property supplies the nested survival events used in this
posterior calculation.

\paragraph{One event space for restricted and full heads.}
Let $\mathcal V$ be the native scoring alphabet, containing every admissible
target.  Introduce a distinct symbol $\bot\notin\mathcal V$ to tag the
persistence event.  On $\{\bot\}\sqcup\mathcal V$, the two distributions in
the native event space are
\begin{equation}
Q^y=(1-\gamma)\delta_{\bot}+\gamma\delta_y,
\qquad
P_\theta=(1-\gamma)\delta_{\bot}+\gamma p_\theta(\cdot\mid X).
\label{eq:native_reverse_pair}
\end{equation}
They are normalized, and
\begin{align}
D_{\mathrm{KL}}(Q^y\|P)
&=(1-\gamma)\log\frac{1-\gamma}{1-\gamma}
  +\gamma\log\frac{\gamma}{\gamma p(y)}
 =-\gamma\log p(y).
\label{eq:tagged_native_reverse_kl}
\end{align}
Consequently,
\begin{equation}
\gamma^{-1}D_{\mathrm{KL}}(Q^y\|P_\theta)
=-\log p_\theta(y\mid X).
\label{eq:native_reverse_kl}
\end{equation}
At $\gamma=1$, the zero-mass persistence term is defined as zero.  If
$p(y)=0$, both sides are infinite.  If the scoring head excludes
$\masktok$, mapping $\bot$ to $\masktok$ is one-to-one on the event space
and yields the ordinary absorbing reverse-token kernel.  If the head contains
$\masktok$, the event space retains the distinction between persistence and
a mask-category prediction.  Thus \Cref{eq:native_reverse_kl} holds for the
native categorical score of either head.  CT-OPD couples this exact score with
the student's confidence-based reveal process, which determines the
trajectory-mask law.

The full-head state marginal follows by merging $\bot$ and the mask category
into the same observed mask state.  Write $\widehat Q^y,\widehat P$ for the
resulting distributions.
For $0<\gamma<1$ and $p(y)>0$,
\begin{align}
D_{\mathrm{KL}}(\widehat Q^y\|\widehat P)
&=-\gamma\log p(y)
 +(1-\gamma)\log
 \frac{1-\gamma}{1-\gamma+\gamma p(\masktok)},\\
D_{\mathrm{KL}}(Q^y\|P)
 -D_{\mathrm{KL}}(\widehat Q^y\|\widehat P)
&=(1-\gamma)\log\!\left(
1+\frac{\gamma p(\masktok)}{1-\gamma}\right)\geq0.
\label{eq:full_head_coarsening_gap}
\end{align}
This is the exact information lost by merging the two events.  The gap is
zero for a mask-excluding head and tends to zero as $\gamma\uparrow1$.
For the same full head, let $r=1-p(\masktok)>0$ and
$\widetilde p(v)=p(v)/r$ for non-mask categories.  Its training score obeys
\begin{equation}
-\log p(y)=-\log r-\log\widetilde p(y).
\label{eq:full_head_score_split}
\end{equation}
The loss therefore trains the normalized token-value prediction and penalizes
probability assigned to the mask category.  When $r=0$, every admissible
non-mask target has infinite log loss.  At $\gamma=0$, both event laws place
all mass on persistence and the normalization by $\gamma$ is undefined; the
reverse-normalized identity is stated for $\gamma>0$.

\subsubsection{Conditional properness and quantitative calibration}

\paragraph{The information available to each native query.}
Write $X_i$ for the information the architecture permits prediction $i$ to
use, including context, the relevant reconstructed tokens, and native
attention or stage metadata.  For a bidirectional head, this can be the
complete scoring input.  For query $i$ in causal block $b$, it comprises the
prompt, clean earlier blocks, the reconstructed current block, and permitted
metadata, as specified in \Cref{eq:bd3_multiblock_objective}.  A native
predictor is a measurable function $p_i(\cdot\mid X_i)$.  This choice makes
$Q_i$ the exact Bayes target under the architecture's native information set.

\begin{proposition}[Conditional properness and calibration]
\label{prop:properness}
Fix the trajectory-mask law at cycle start and a finite prediction alphabet
containing the targets.  Let $X_i$ be the information available to native
prediction $i$, set
$w_i=\mathbf1\{i\in U\}/(|U|\vee1)$ and
$\bar w_i(X_i)=\E[w_i\mid X_i]$, and define
\begin{equation}
Q_i(v\mid X_i)
=\frac{\E[w_i\mathbf1\{Y_i^\teacher=v\}\mid X_i]}
{\bar w_i(X_i)}
\label{eq:appendix_weighted_target}
\end{equation}
where $\bar w_i(X_i)>0$.  For finite risk, the excess over all predictors
permitted by the same native conditioning information is
\begin{equation}
\cL_{\mathrm{CT}}(p_\theta)-\cL_{\mathrm{CT}}^*
=\E\sum_i\bar w_i(X_i)
D_{\mathrm{KL}}\!\left(Q_i(\cdot\mid X_i)\|p_{\theta,i}(\cdot\mid X_i)\right).
\label{eq:weighted_properness}
\end{equation}
Thus $Q_i$ is the unique optimum wherever $\bar w_i(X_i)>0$, up to null
sets, and excess CT risk controls weighted conditional-distribution error and
excess Bayes token-decision risk.
\end{proposition}

\begin{proof}
Set
\begin{equation}
w_i=\frac{\mathbf1\{i\in U\}}{|U|\vee1},\qquad
\bar w_i(X_i)=\E[w_i\mid X_i].
\label{eq:properness_weight}
\end{equation}
These weights are nonnegative, bounded by one, and satisfy
$\sum_iw_i=\mathbf1\{|U|>0\}$.  Extend $Y_i$ outside the active span by any
fixed vocabulary symbol; those extensions receive zero weight.  For
$\bar w_i(X_i)>0$, define
\begin{equation}
Q_i(v\mid X_i)=\frac{\E[w_i\mathbf1\{Y_i=v\}\mid X_i]}
{\bar w_i(X_i)}.
\label{eq:weighted_native_target}
\end{equation}
Summing its numerator over the finite scoring alphabet gives
$\bar w_i(X_i)$, so $Q_i$ is a probability distribution.  On
$\{\bar w_i=0\}$, choose any fixed categorical distribution.  Nonnegativity
implies $w_i=0$ almost surely on this event, which contributes zero loss.

Conditioning each summand on its own $X_i$ and interchanging finite sums and
nonnegative expectations gives
\begin{align}
\cL_{\mathrm{CT}}(p)
&=\E\sum_i\bar w_i(X_i)
 \sum_v Q_i(v\mid X_i)(-\log p_i(v\mid X_i))\\
&=\E\sum_i\bar w_i(X_i)
 \left[H(Q_i(\cdot\mid X_i))+
 D_{\mathrm{KL}}(Q_i(\cdot\mid X_i)\|p_i(\cdot\mid X_i))\right].
\end{align}
For a vocabulary of size $V$, the entropy term is bounded by
$\log V\,\E\sum_i\bar w_i\leq\log V$.  The measurable choice $p_i=Q_i$
attains this term for all queries simultaneously in the class of
native-context categorical predictors.  Subtracting that Bayes risk proves
\Cref{eq:weighted_properness}.

For finite distributions, KL is nonnegative, and it is zero exactly when
the distributions coincide.  A nonnegative random variable has zero
expectation only if it is zero almost surely.  Therefore a minimizing
predictor equals $Q_i$ almost surely wherever $\bar w_i>0$.  Conversely,
those equalities imply optimal risk.  Uniqueness is distributional: it
identifies the Bayes-optimal conditional prediction $Q_i$ wherever
$\bar w_i>0$, with equivalent parameterizations representing the same optimum.

Using natural logarithms, Pinsker's inequality gives
$\operatorname{TV}(Q_i,p_i)^2\leq\tfrac12D_{\mathrm{KL}}(Q_i\|p_i)$.
Multiplication by $\bar w_i$, summation, and expectation give the quantitative
calibration bound
\begin{equation}
\E\sum_i\bar w_i(X_i)
\operatorname{TV}^2(Q_i,p_i)
\leq\tfrac12\bigl(\cL_{\mathrm{CT}}(p)-\cL_{\mathrm{CT}}^*\bigr).
\label{eq:ct_calibration}
\end{equation}
Zero predicted mass on a target with positive conditional mass gives infinite
risk.  The finite-risk hypothesis makes the excess-risk identity well defined,
and the result follows directly from the conditional KL decomposition.
\end{proof}

The weighting remains necessary even if an active endpoint length is hidden
from a particular query.  For example, consider two equally likely examples
that produce the same $X_i$, with $i$ supervised in both and unresolved counts
one and two.  If their targets at $i$ are distinct, their weights are $1$ and
$1/2$, so the conditional target assigns probabilities $2/3$ and $1/3$.
When $w_i$ is determined by $X_i$, \Cref{eq:weighted_native_target} reduces
to the ordinary conditional endpoint-token law on the supervised queries.
The result holds under arbitrary dependence among $X_i$, $Y_i$, and $w_i$.

\paragraph{Calibration of token decisions.}
\label{app:token_decision_calibration}
Put $\kappa=\Pr(|U|>0)=\E\sum_i\bar w_i$ and
$\Delta=\cL_{\mathrm{CT}}(p)-\cL_{\mathrm{CT}}^*$.  Let
$\widehat v_i$ maximize $p_i(\cdot\mid X_i)$, with a fixed rule to break
ties, and let $v_i^*$ maximize $Q_i(\cdot\mid X_i)$.  For the weighted token
misclassification risk
$\mathcal R_{01}(\widehat v)=\E\sum_iw_i\mathbf1\{\widehat v_i\ne Y_i\}$,
the excess over its Bayes value satisfies
\begin{equation}
\mathcal R_{01}(\widehat v)-\mathcal R_{01}^*
\leq\sqrt{2\kappa\Delta}.
\label{eq:token_decision_calibration}
\end{equation}
To prove this, use $p_i(\widehat v_i)\geq p_i(v_i^*)$ to obtain
\begin{align}
Q_i(v_i^*)-Q_i(\widehat v_i)
&\leq |Q_i(v_i^*)-p_i(v_i^*)|
     +|Q_i(\widehat v_i)-p_i(\widehat v_i)|\\
&\leq 2\operatorname{TV}(Q_i,p_i).
\end{align}
If the two maximizing categories coincide, the left side is zero; otherwise
the two absolute differences are part of the total $\ell_1$ difference.
After weighting and averaging, Cauchy-Schwarz gives
$\E\sum_i\bar w_i\operatorname{TV}(Q_i,p_i)
\leq[\kappa\E\sum_i\bar w_i\operatorname{TV}^2(Q_i,p_i)]^{1/2}$.
Combining this with \Cref{eq:ct_calibration} proves the claim.
If $\kappa=0$, every weighted risk is zero.  If $\kappa>0$, dividing by
$\kappa$ yields the normalized bounds on the supervised-token population,
including $\kappa^{-1}\E\sum_i\bar w_i\operatorname{TV}(Q_i,p_i)
\leq\sqrt{\Delta/(2\kappa)}$.

\subsubsection{Token gradients with equal terminal rewards}
\label{app:starvation_gradient}

\begin{proposition}[Token signal with equal terminal rewards]
\label{prop:terminal_starvation}
If $r_1=\cdots=r_G$, every centered advantage
$A_k=(r_k-\bar r)/(\sigma_r+\epsilon)$, $\epsilon>0$, vanishes.  For a CT-OPD
state with $N=|U|>0$ and native logits $z_{i,v}$,
\begin{equation}
\frac{\partial\ell_{\mathrm{CT}}}{\partial z_{i,v}}
=\frac{p_{\theta,i}(v\mid X_i)-\mathbf1\{v=Y_i^\teacher\}}{N},
\qquad
\|\nabla_{z_i}\ell_{\mathrm{CT}}\|_1
=\frac{2(1-p_{\theta,i}(Y_i^\teacher\mid X_i))}{N}.
\label{eq:ct_noncontrast_risk}
\end{equation}
Every supervised token with
$p_{\theta,i}(Y_i^\teacher\mid X_i)<1$ therefore has a nonzero gradient with
respect to its native logits when all sampled terminal rewards agree.
\end{proposition}

\begin{proof}[Proof of \Cref{prop:terminal_starvation}]
If all rewards equal $r$, their mean is $r$, their standard deviation is zero,
and each centered advantage is $0/(0+\epsilon)=0$.  Every task-gradient term
multiplied by that advantage is therefore zero, including a clipped
importance-ratio term.  Auxiliary regularizers remain separate from this task
signal.

For a fixed CT example with $N=|U|>0$, use finite logits over the native
scoring alphabet and write the contribution of supervised token $i$ as
$\ell_i=N^{-1}[-z_{i,Y_i}+\log\sum_v\exp z_{i,v}]$.
Differentiating gives the first identity in
\Cref{eq:ct_noncontrast_risk}.  Its target component has magnitude
$(1-p_i(Y_i))/N$; its non-target components are nonnegative and sum to the
same quantity.  Adding these magnitudes proves the second identity.  It is
strictly positive when $p_i(Y_i)<1$, regardless of terminal rewards.  An
empty target set contributes zero, and a one-category head is already exact.
\end{proof}

\subsection{Trajectory-mask allocation and reconstructed context}

\subsubsection{Trace allocation at matched cardinality}
\label{app:trace_allocation}

Fix the context, endpoint $y$, scorer, and count
$1\leq n\leq L_y=|A(y)|$, with positive probability for the trace count
being analyzed.  All expectations below condition on these objects.  Let
$\pi_i^{\mathrm{tr}}$ be the trace inclusion probability and
$d_i^{\mathrm{tr}}$ the expected token loss conditional on inclusion.  If
$\pi_i^{\mathrm{tr}}=0$, choose any finite value for
$d_i^{\mathrm{tr}}$, whose coefficient is zero.  Assume the conditional
losses are integrable.  Uniform size-$n$ active masking includes each
position with probability
$\binom{L_y-1}{n-1}/\binom{L_y}{n}=n/L_y$, giving
\begin{align}
R_{\mathrm{tr}}^{(n)}&=\frac1n\sum_i\pi_i^{\mathrm{tr}}d_i^{\mathrm{tr}},
& R_{\mathrm{rand}}^{(n)}&=\frac1{L_y}\sum_i d_i^{\mathrm{rand}}.
\end{align}
Their difference has the exact decomposition
\begin{equation}
R_{\mathrm{tr}}^{(n)}-R_{\mathrm{rand}}^{(n)}
=\frac{L_y}{n}\operatorname{Cov}_{i}
(\pi_i^{\mathrm{tr}},d_i^{\mathrm{rand}})
+\frac1n\sum_i\pi_i^{\mathrm{tr}}
(d_i^{\mathrm{tr}}-d_i^{\mathrm{rand}}).
\label{eq:trace_allocation}
\end{equation}
Since each trace set has size $n$, summing its inclusion indicators before
expectation gives $\sum_i\pi_i^{\mathrm{tr}}=n$.  Add and subtract
$n^{-1}\sum_i\pi_i^{\mathrm{tr}}d_i^{\mathrm{rand}}$.  With covariance
taken over a uniformly selected active position,
\begin{align}
\frac1n\sum_i\pi_i^{\mathrm{tr}}d_i^{\mathrm{rand}}
-\frac1{L_y}\sum_i d_i^{\mathrm{rand}}
&=\frac{L_y}{n}\left[
 \frac1{L_y}\sum_i\pi_i^{\mathrm{tr}}d_i^{\mathrm{rand}}
 -\frac{n}{L_y}\frac1{L_y}\sum_i d_i^{\mathrm{rand}}\right]\\
&=\frac{L_y}{n}\operatorname{Cov}_i
 (\pi_i^{\mathrm{tr}},d_i^{\mathrm{rand}}).
\end{align}
The remaining term is
$n^{-1}\sum_i\pi_i^{\mathrm{tr}}(d_i^{\mathrm{tr}}-d_i^{\mathrm{rand}})$,
which proves \Cref{eq:trace_allocation}.  The second term retains context
changes induced by the joint mask law even if the two laws have identical
marginal inclusion probabilities.

At $n=L_y$, the active subset is unique.  Matching native metadata across the
two arms makes their reconstructed contexts identical and every inclusion
probability equal to one, so both terms vanish exactly.  Full masking is
therefore the zero-allocation boundary; at partial mask levels, the
mask-allocation and joint-context terms in \Cref{eq:trace_allocation} may be
nonzero.

For variable counts, set $R^{(0)}=0$.  Given trace and random count laws
$q_{\mathrm{tr}}$ and $q_{\mathrm{rand}}$, their risks satisfy
\begin{align}
R_{\mathrm{tr}}-R_{\mathrm{rand}}
={}&\sum_n q_{\mathrm{tr}}(n)
 (R_{\mathrm{tr}}^{(n)}-R_{\mathrm{rand}}^{(n)})\nonumber\\
&+\sum_n(q_{\mathrm{tr}}(n)-q_{\mathrm{rand}}(n))R_{\mathrm{rand}}^{(n)}.
\label{eq:trace_count_mixture}
\end{align}
Terms with $q_{\mathrm{tr}}(n)=0$ have zero contribution in the first sum.
The second sum vanishes under a common count law.  The Random baseline
preserves each collected trace's active count when randomizing positions.  The
identity decomposes the collection-scorer risk gap into mask-allocation and
joint-context terms; the final benchmark comparison measures the trained
checkpoints.

\subsubsection{Finite-family value of reconstructed context}
\label{app:context_information}

The allocation identity compares a trajectory-mask law with a reference mask
law at a fixed scorer.  We next separate the Bayes information in reconstructed context from the ability
of a fixed predictor family to use that context.  Assume
$\kappa=\Pr(|U|>0)>0$ and define the supervised-token probability measure
\begin{equation}
\E_{\dagger}[h(I,X_I,T)]
=\frac1\kappa\E\sum_iw_i h(i,X_i,Y_i)
\label{eq:supervised_token_measure}
\end{equation}
for bounded measurable $h$.  This first conditions on a nonempty example and
then samples one unresolved active coordinate uniformly.

Let $\mathsf F=(I,X_I)$ be the full native query information and let
$\mathsf B=b(\mathsf F)$ retain the prompt, query coordinate, and common
architecture metadata while omitting the reconstructed visible response
tokens.  For predictor families $\mathcal H_{\mathsf F}$ and
$\mathcal H_{\mathsf B}$, denote their optimal unnormalized weighted log
losses by $\mathcal L_{\mathcal H_{\mathsf F}}^*$ and
$\mathcal L_{\mathcal H_{\mathsf B}}^*$.  Define their normalized
approximation gaps above the corresponding Bayes risks by
\begin{align}
a_{\mathsf F}
&=\kappa^{-1}\mathcal L_{\mathcal H_{\mathsf F}}^*
  -H_{\dagger}(T\mid\mathsf F),\nonumber\\
a_{\mathsf B}
&=\kappa^{-1}\mathcal L_{\mathcal H_{\mathsf B}}^*
  -H_{\dagger}(T\mid\mathsf B).
\label{eq:finite_family_approximation_gaps}
\end{align}
Properness gives $a_{\mathsf F},a_{\mathsf B}\geq0$.  Direct subtraction yields
the exact finite-family decomposition
\begin{equation}
\mathcal L_{\mathcal H_{\mathsf B}}^*
-\mathcal L_{\mathcal H_{\mathsf F}}^*
=\kappa\!\left[
I_{\dagger}(T;\mathsf F\mid\mathsf B)
+a_{\mathsf B}-a_{\mathsf F}\right].
\label{eq:context_information_gain}
\end{equation}
The first term is the additional Bayes information in the native query; the
second is the additional approximation advantage available to the chosen
finite predictor family.  In the deterministic-endpoint regime,
$I_{\dagger}(T;\mathsf F\mid\mathsf B)=0$ and
\Cref{eq:context_information_gain} reduces exactly to
$\kappa(a_{\mathsf B}-a_{\mathsf F})$; reconstructed visible tokens improve
attainable loss whenever $a_{\mathsf F}<a_{\mathsf B}$.
\Cref{eq:trace_allocation} separates the fixed-scorer mask comparison into
position allocation and the contexts induced by the two mask laws.
\Cref{eq:context_information_gain} separately characterizes what reconstructed
visible response context contributes to the best predictor in a chosen family:
Bayes information and finite-family approximation advantage.  For $\kappa=0$,
the unnormalized risks are all zero and no normalized token measure is required.

\subsubsection{Transporting native difficulty to teacher reconstruction}
\label{app:difficulty_transport}

\paragraph{Adjacent nested masks: identity and component bound.}
The projected unresolved-coordinate sets from adjacent states in one monotone
rollout satisfy
$K_j:=U_{j+1}\subset U_j\subseteq A(y)$.  Write
$m_j=|U_j|$, $n_j=|K_j|$, $R_j=U_j\setminus K_j$, and
$r_j=|R_j|$, with $0<n_j<m_j$.  For
$\overline q_B=|B|^{-1}\sum_{i\in B}q_i$, define
\begin{equation}
\mathfrak M_j(q)=\overline q_{K_j}-\overline q_{U_j}.
\label{eq:adjacent_trajectory_mask_margin}
\end{equation}
Let $S_j$ be the raw parent state and let $X_j=\cC(y,U_j)$ be the
endpoint-reconstructed parent state induced by its projected trajectory mask.
At the fixed collection scorer, put
\begin{align}
h_{j,i}&=-\log\max_v p_{\theta,i}(v\mid S_j),
&a_{j,i}&=\log\frac{\max_v p_{\theta,i}(v\mid S_j)}
 {p_{\theta,i}(y_i\mid S_j)},\nonumber\\
d_{j,i}&=-\log p_{\theta,i}(y_i\mid X_j),
&\delta_{j,i}&=d_{j,i}-h_{j,i}-a_{j,i}.
\label{eq:adjacent_transport_scores}
\end{align}
Set
\begin{equation}
(\mathcal H_j,\mathcal A_j,\mathcal E_j,\mathcal T_j)
=(\mathfrak M_j(h_j),\mathfrak M_j(a_j),
  \mathfrak M_j(\delta_j),\mathfrak M_j(d_j)).
\label{eq:adjacent_transport_terms}
\end{equation}
Restricting the raw masks to the teacher endpoint's active coordinates
preserves nesting.  The nontrivial retained-versus-revealed domain is exactly
$0<n_j<m_j$.

\begin{proposition}[Observable transfer across adjacent masks]
\label{prop:difficulty_transport}
If $V_j$ is a uniformly random $n_j$-subset of $U_j$, independent of the fixed
score vector $q$, then
\begin{equation}
\mathfrak M_j(q)
=\E_{V_j}[\overline q_{K_j}-\overline q_{V_j}]
=\frac{r_j}{m_j}(\overline q_{K_j}-\overline q_{R_j}).
\label{eq:adjacent_random_identity}
\end{equation}
Moreover,
\begin{equation}
\mathcal T_j
=\E_{V_j}[\overline d_{K_j}-\overline d_{V_j}]
=\mathcal H_j+\mathcal A_j+\mathcal E_j.
\label{eq:adjacent_transport_certificate}
\end{equation}
Thus $\mathcal T_j>0$ means exactly that reconstructed teacher-target loss is
higher on the student-retained coordinates than on the expected count-matched
random subset at the same parent query.
\end{proposition}

For the same comparison, define the correction budget
\begin{align}
\mathcal B_j
&:=\sum_{i\in U_j}
\left|\frac{\mathbf1\{i\in K_j\}}{n_j}-\frac1{m_j}\right|
|\delta_{j,i}|\nonumber\\
&=\frac{r_j}{m_j}\left[
 \frac1{n_j}\sum_{i\in K_j}|\delta_{j,i}|
 +\frac1{r_j}\sum_{i\in R_j}|\delta_{j,i}|
\right].
\label{eq:adjacent_transport_budget_groups}
\end{align}

\begin{proof}
For every $i\in U_j$, uniform sampling without replacement gives
$\Pr(i\in V_j)=n_j/m_j$.  Therefore
\begin{equation}
\E_{V_j}\!\left[\frac1{n_j}\sum_{i\in V_j}q_i\right]
=\frac1{n_j}\sum_{i\in U_j}\frac{n_j}{m_j}q_i
=\frac1{m_j}\sum_{i\in U_j}q_i.
\label{eq:adjacent_random_mean}
\end{equation}
This proves the first equality in \Cref{eq:adjacent_random_identity}.  Since
$U_j=U_{j+1}\mathbin{\dot\cup}R_j$,
\begin{align}
\mathfrak M_j(q)
&=\overline q_{U_{j+1}}
-\frac{n_j\overline q_{U_{j+1}}
      +|R_j|\overline q_{R_j}}{m_j}\nonumber\\
&=\frac{|R_j|}{m_j}
 \left(\overline q_{U_{j+1}}-\overline q_{R_j}\right),
\label{eq:adjacent_retained_revealed}
\end{align}
which proves the second equality.

The softmax definitions in \Cref{eq:adjacent_transport_scores} give,
coordinate by coordinate,
\begin{equation}
-\log p_{\theta,i}(y_i\mid S_j)=h_{j,i}+a_{j,i},
\qquad
d_{j,i}=h_{j,i}+a_{j,i}+\delta_{j,i}.
\label{eq:adjacent_raw_target_identity}
\end{equation}
The functional $\mathfrak M_j$ is linear, so applying it to the second identity
proves the equality in \Cref{eq:adjacent_transport_certificate}.  Put
$b_{j,i}=\mathbf1\{i\in K_j\}/n_j-1/m_j$.
Then $\mathcal E_j=\sum_{i\in U_j}b_{j,i}\delta_{j,i}$, and the triangle
inequality gives
\begin{equation}
\mathcal E_j
\geq-\sum_{i\in U_j}|b_{j,i}|\,|\delta_{j,i}|
=-\mathcal B_j.
\label{eq:adjacent_error_bound}
\end{equation}
Evaluating the two constant absolute weights on $K_j$ and $R_j$ gives
\Cref{eq:adjacent_transport_budget_groups}; combining
$\mathcal E_j\geq-\mathcal B_j$ with
\Cref{eq:adjacent_transport_certificate} yields the corresponding componentwise
lower bound.
\end{proof}

Here $\mathcal T_j$ is the signed transport statistic.  On the frozen panel,
the positive reconstruction term $\mathcal E$ accounts for nearly all of the
pooled margin, and reconstructed teacher-target loss is higher on retained
positions than under the count-matched expectation.  Because $\mathcal E$ is
signed, this contribution is absent from the componentwise absolute envelope
$\mathcal H_j+\mathcal A_j-\mathcal B_j$.

The same proof applies to any finite coordinate score.  In particular,
$c_{j,i}=2(1-p_{\theta,i}(y_i\mid X_j))$ is the per-coordinate contribution to
the stacked $\ell_1$ logit-gradient magnitude before normalization by the
number of supervised positions.  Thus $\mathfrak M_j(c_j)$ is exactly its
student-retained mean minus the expected mean of a count-matched random subset
scored on the same reconstructed parent.  This identity measures allocation on
a common reconstructed parent, while the Random reconstruction provides
the accompanying context comparison in the full-risk evaluation.

\IfFileExists{generated/CT_OPD_W3_PAPER_TABLES_COMPLETE.tex}{%
\begin{table}[H]
\centering
\caption{\textbf{Difficulty remains concentrated at trajectory-selected positions.}
All quantities use the frozen Base scorer.  Each measured margin reports a
point estimate and prompt-paired 95\% CI.  The pooled row first
averages available $75\%\!\to\!50\%$ and $50\%\!\to\!25\%$ transitions within
each prompt.  $\mathcal T$ is the endpoint-target NLL allocation margin over
the exact uniform count-matched subset expectation at the same reconstructed
parent.  The signed components satisfy $\mathcal T=\mathcal H+\mathcal A+\mathcal E$.}
\label{tab:w3_difficulty_transport}
\arrayrulecolor{tableline}
\scriptsize
\setlength{\tabcolsep}{2.2pt}
\renewcommand{\arraystretch}{1.16}
\resizebox{\textwidth}{!}{%
\begin{tabular}{@{}lrrrrr@{}}
\toprule
Scope & $n$ & $\mathcal H$ & $\mathcal A$ & $\mathcal E$ & $\mathcal T$ \\
\midrule
Pooled visible parent & 445 & $0.2975\;[\,0.2795,\,0.3171\,]$ & $-0.2828\;[\,-0.4207,\,-0.1387\,]$ & $1.0976\;[\,0.9876,\,1.2041\,]$ & $1.1123\;[\,0.9325,\,1.3053\,]$ \\
$75\%\!\to\!50\%$ & 445 & $0.2655\;[\,0.2450,\,0.2874\,]$ & $-0.2274\;[\,-0.3634,\,-0.0915\,]$ & $1.0228\;[\,0.9158,\,1.1337\,]$ & $1.0609\;[\,0.8734,\,1.2490\,]$ \\
$50\%\!\to\!25\%$ & 354 & $0.2948\;[\,0.2733,\,0.3173\,]$ & $-0.5834\;[\,-0.7441,\,-0.4169\,]$ & $1.0328\;[\,0.8845,\,1.1787\,]$ & $0.7442\;[\,0.5529,\,0.9306\,]$ \\
\bottomrule
\end{tabular}%
}
\vspace{2pt}

\begin{tabular}{@{}lr@{}}
\toprule
Scope & Mean token-gradient magnitude margin \\
\midrule
Pooled visible parent & $0.1050\;[\,0.0907,\,0.1187\,]$ \\
$75\%\!\to\!50\%$ & $0.0947\;[\,0.0820,\,0.1073\,]$ \\
$50\%\!\to\!25\%$ & $0.1079\;[\,0.0842,\,0.1302\,]$ \\
\bottomrule
\end{tabular}
\vspace{1pt}

\textcolor{tablemuted}{\scriptsize Token-gradient magnitude is the mean unnormalized categorical-logit $\ell_1$ norm.\quad Identity control
$100\%\!\to\!75\%$: residual $-1.498\!\times\!10^{-17}\;[\,-6.432\!\times\!10^{-17},\,3.569\!\times\!10^{-17}\,]$.}
\arrayrulecolor{black}
\end{table}
}{}

\paragraph{General fixed-count mask laws.}
Reconstruction combines trace-derived mask allocation with visible
teacher values.  The following comparison relates the raw confidence profile
to teacher-target difficulty after reconstruction.
The first comparison holds a context $c$, endpoint $y$, collection scorer
$\theta=\bar\theta$, and count $1\leq n\leq L_y$ fixed, with positive trace
probability for that count.  For a native prediction domain $D$, use the local
active set $A_D(y)=A(y)\cap D$ and let $L_y=|A_D(y)|$ throughout this section.
For block models, we condition on the block and required metadata and form the
reference masks over its active coordinates.  Explicit mixture weights in
\Cref{eq:trace_count_mixture} aggregate the resulting domain- and count-specific
comparisons.

For each active coordinate $i$, specify a raw proposal query $S_i$ drawn from
the recorded student trace at a decision where $i$ is unresolved.  We use the
last native proposal for $i$ at or before the selected trace stage, so every
compared coordinate has a recorded masked proposal.
An already revealed coordinate is probed at its recorded earlier proposal,
when its own value was still masked.  Final inclusion probabilities
$\pi_i^{\mathrm{tr}}$ and modal uncertainty $h_i$ are evaluated separately,
allowing arbitrary history-dependent selection rules.  The same scorer is
used throughout the comparison.  Re-evaluating stored queries at
$\theta\ne\bar\theta$ gives the corresponding bound for the later model's
uncertainty profile.

Couple $S_i$ to a reconstructed query $X_i^{\mathrm{rand}}$ whose marginal
law is uniform size-$n$ active masking conditional on including $i$, with the
same endpoint and native reference metadata used to define
$d_i^{\mathrm{rand}}$ in \Cref{app:trace_allocation}.  Any fixed coupling
with these marginals is admissible.  With finite logits, define
\begin{align}
h_i&=\E[-\log\max_v p_{\theta,i}(v\mid S_i)],\nonumber\\
a_i&=\E\log\frac{\max_v p_{\theta,i}(v\mid S_i)}
 {p_{\theta,i}(y_i\mid S_i)},\nonumber\\
\eta_i&=\E\operatorname{osc}_v
 [z_{i,v}(X_i^{\mathrm{rand}})-z_{i,v}(S_i)].
\label{eq:general_transport_scores}
\end{align}
In particular,
\begin{equation}
\E[-\log p_{\theta,i}(y_i\mid S_i)]=h_i+a_i,
\qquad a_i\geq0.
\label{eq:raw_uncertainty_target_gap}
\end{equation}
The target-gap term $a_i$ records student-teacher disagreement separately from
modal uncertainty, while $\eta_i$ measures the query shift induced by
reconstruction.  All logits are finite, and all displayed expectations are
assumed finite.  These expectations define the prediction-difficulty profile
for the fixed collection law and chosen reference coupling.

\begin{proposition}[General fixed-count difficulty transport]
\label{prop:difficulty_transport_general}
For any common offset $a_0$, set $\epsilon_i=|a_i-a_0|+\eta_i$.  Then
\begin{equation}
\frac{L_y}{n}\operatorname{Cov}_i
 (\pi_i^{\mathrm{tr}},d_i^{\mathrm{rand}})
\geq\frac{L_y}{n}\operatorname{Cov}_i
 (\pi_i^{\mathrm{tr}},h_i)
-\frac1n\sum_i\left|\pi_i^{\mathrm{tr}}-\frac n{L_y}\right|\epsilon_i.
\label{eq:difficulty_transport_general}
\end{equation}
\end{proposition}

\begin{proof}[Proof of \Cref{prop:difficulty_transport_general}]
For a paired query, write
$r_v=z_{i,v}(X_i^{\mathrm{rand}})-z_{i,v}(S_i)$.  The softmax definition gives
\begin{equation}
\log\frac{p_{\theta,i}(y_i\mid X_i^{\mathrm{rand}})}
 {p_{\theta,i}(y_i\mid S_i)}
=r_{y_i}-\log\sum_v p_{\theta,i}(v\mid S_i)e^{r_v}.
\label{eq:transport_log_probability}
\end{equation}
The sum is a convex combination of $e^{r_v}$, so its logarithm lies between
$\min_v r_v$ and $\max_v r_v$.  Therefore the absolute log-probability
change is at most $\operatorname{osc}_v r_v$.  Taking expectations and using
\Cref{eq:raw_uncertainty_target_gap} yields
\begin{equation}
|d_i^{\mathrm{rand}}-h_i-a_i|\leq\eta_i,
\qquad |d_i^{\mathrm{rand}}-h_i-a_0|\leq|a_i-a_0|+\eta_i=\epsilon_i.
\label{eq:transport_profile_bound}
\end{equation}
Define $b_i=(\pi_i^{\mathrm{tr}}-n/L_y)/n$.  Since each active mask has size
$n$, $\sum_i\pi_i^{\mathrm{tr}}=n$ and $\sum_i b_i=0$.  With uniform
coordinate covariance, write
\begin{align}
A_{\mathrm S}&=\frac{L_y}{n}\operatorname{Cov}_i(\pi_i^{\mathrm{tr}},h_i),
&A_{\mathrm T}&=\frac{L_y}{n}\operatorname{Cov}_i
 (\pi_i^{\mathrm{tr}},d_i^{\mathrm{rand}}),\nonumber\\
A_{\mathrm T}-A_{\mathrm S}
&=\sum_i b_i(d_i^{\mathrm{rand}}-h_i-a_0)
\geq-\sum_i|b_i|\epsilon_i.
\label{eq:transport_centered_allocation}
\end{align}
This is \Cref{eq:difficulty_transport_general}.  A common target gap cancels;
only its variation across positions enters the allocation correction.
\end{proof}

\paragraph{Margin and uniform guarantees.}
\Cref{eq:transport_profile_bound} also gives the pairwise implication
\begin{equation}
h_i-h_k>\epsilon_i+\epsilon_k
\quad\Longrightarrow\quad d_i^{\mathrm{rand}}>d_k^{\mathrm{rand}}.
\label{eq:transport_rank_margin}
\end{equation}
Thus every raw uncertainty ranking separated by the explicit perturbation
margin transfers through reconstruction.  If all
$\epsilon_i\leq\epsilon$, then
\begin{equation}
A_{\mathrm T}\geq A_{\mathrm S}-2(1-n/L_y)\epsilon.
\label{eq:transport_uniform_bound}
\end{equation}
Indeed, put $q=n/L_y$.  For $0\leq x\leq1$, convexity gives
$|x-q|\leq(1-x)q+x(1-q)$.  Summing with $x=\pi_i^{\mathrm{tr}}$ gives
$\sum_i|\pi_i^{\mathrm{tr}}-q|\leq2n(1-q)$.  Full masking has $n=L_y$ and
$\pi_i^{\mathrm{tr}}=1$, so both allocation terms are zero.  Uniform size-$n$
masking likewise has $\pi_i^{\mathrm{tr}}=n/L_y$ and zero marginal allocation.
The conditional comparison covers $n\geq1$, while empty masks contribute zero
to the original objective.  Reference scores and centered coefficients remain
well defined at coordinates with zero trace inclusion.

\paragraph{Separating value replacement from the reference mask law.}
When the native adapter supports it, introduce $X_i^{\mathrm{swap}}$ by
reconstructing the teacher endpoint at the raw proposal's own mask.  Keep
the query coordinate unresolved and use the architecture's prescribed endpoint
prefix and inactive-position convention.  On the same joint coupling define
\begin{align}
\eta_i^{\mathrm{value}}
&=\E\operatorname{osc}_v[z_{i,v}(X_i^{\mathrm{swap}})-z_{i,v}(S_i)],\nonumber\\
\eta_i^{\mathrm{reference}}
&=\E\operatorname{osc}_v[z_{i,v}(X_i^{\mathrm{rand}})-z_{i,v}(X_i^{\mathrm{swap}})],
&\eta_i&\leq\eta_i^{\mathrm{value}}+\eta_i^{\mathrm{reference}}.
\label{eq:transport_drift_split}
\end{align}
The inequality follows from subadditivity of maximum-minus-minimum.
The first term includes replacement of visible values and any required
endpoint-domain or prefix adaptation.  The second measures the move to the
reference mask context, including differences in count or decision stage.
Thus the total $\eta_i$ covers both value replacement and reference-query
changes.  Every admissible coupling obeys the paired bound, and an
adapter-compatible intermediate query yields the displayed value/reference
decomposition.

\paragraph{Sufficient condition for difficulty allocation.}
The sufficient condition
$A_{\mathrm S}>n^{-1}\sum_i|\pi_i^{\mathrm{tr}}-n/L_y|\epsilon_i$
implies that, after reconstruction, the trace-derived mask law assigns more
supervision to positions with higher teacher-target loss than uniform size-$n$
masking.  Its three terms are measurable under the collection and evaluation
laws: native allocation, target disagreement, and query perturbation.  The
matched-count risk difference in \Cref{eq:trace_allocation} adds the joint-context term
$n^{-1}\sum_i\pi_i^{\mathrm{tr}}(d_i^{\mathrm{tr}}-d_i^{\mathrm{rand}})$,
while \Cref{eq:difficulty_transport_general} bounds the allocation component.
Together, these terms describe the matched-count mask comparison evaluated
by the controlled training interventions.

\end{document}